\documentclass[onefignum,onetabnum]{siamonline220329}

\usepackage{amsmath,amsfonts,amssymb}
\usepackage{bbm}
\usepackage{graphicx}
\usepackage{epstopdf}
\usepackage{algorithmic}
\ifpdf
  \DeclareGraphicsExtensions{.eps,.pdf,.png,.jpg}
\else
  \DeclareGraphicsExtensions{.eps}
\fi

\usepackage{enumitem}
\setlist[enumerate]{leftmargin=.5in}
\setlist[itemize]{leftmargin=.5in}

\newsiamremark{remark}{Remark}
\newsiamremark{hypothesis}{Hypothesis}
\crefname{hypothesis}{Hypothesis}{Hypotheses}
\newsiamthm{claim}{Claim}

\DeclareMathOperator*{\argmin}{argmin}

\def \bE {\mathbb{E}}

\def \bN {\mathbb{N}}

\def \bP {\mathbb{P}}

\def \bR {\mathbb{R}}
\def \bS {\mathbb{S}}

\def \bZ {\mathbb{Z}}

\def \cC {\mathcal{C}}
\def \cD {\mathcal{D}}
\def \cE {\mathcal{E}}
\def \cF {\mathcal{F}}

\def \cH {\mathcal{H}}

\def \cL {\mathcal{L}}
\def \cM {\mathcal{M}}
\def \cN {\mathcal{N}}
\def \cO {\mathcal{O}}

\def \cR {\mathcal{R}}
\def \cS {\mathcal{S}}

\def \NN {\mathcal{NN}}
\def \CNN {\mathcal{CNN}}

\def \Ba {{\boldsymbol{a}}}
\def \Bb {{\boldsymbol{b}}}

\def \Bf {{\boldsymbol{f}}}

\def \Bs {{\boldsymbol{s}}}

\def \Bv {{\boldsymbol{v}}}
\def \Bw {{\boldsymbol{w}}}
\def \Bx {{\boldsymbol{x}}}
\def \By {{\boldsymbol{y}}}

\def \Bvarphi {{\boldsymbol{\varphi}}}
\def \Btheta {{\boldsymbol{\theta}}}

\def \BA {{\boldsymbol{A}}}
\def \BX {{\boldsymbol{X}}}

\def \sgn {\,{\rm sgn}\,}
\def \Conv {\,{\rm Conv}\,}
\def \Poly {\,{\rm Poly}\,}

\headers{Rates for Convolutional Neural Networks}{Y. Yang, H. Feng, and D.-X. Zhou}

\title{On the rates of convergence for learning with convolutional neural networks}

\author{Yunfei Yang\thanks{School of Mathematics (Zhuhai) and Guangdong Province Key Laboratory of Computational Science, Sun Yat-sen University, Zhuhai, P.R. China
  (\email{yangyunfei@mail.sysu.edu.cn}).}
\and Han Feng\thanks{Department of Mathematics, City University of Hong Kong, Kowloon, Hong Kong
  (\email{hanfeng@cityu.edu.hk}).}
\and Ding-Xuan Zhou\thanks{School of Mathematics and Statistics, University of Sydney, Sydney, NSW 2006, Australia 
  (\email{dingxuan.zhou@sydney.edu.au}).}
}

\usepackage{amsopn}

\ifpdf
\hypersetup{
  pdftitle={On the rates of convergence for learning with convolutional neural networks},
  pdfauthor={Yunfei Yang, Han Feng, and Ding-Xuan Zhou}
}
\fi

\begin{document}

\maketitle

\begin{abstract}
We study approximation and learning capacities of convolutional neural networks (CNNs) with one-side zero-padding and multiple channels. Our first result proves a new approximation bound for CNNs with certain constraint on the weights. Our second result gives new analysis on the covering number of feed-forward neural networks with CNNs as special cases. The analysis carefully takes into account the size of the weights and hence gives better bounds than the existing literature in some situations. Using these two results, we are able to derive rates of convergence for estimators based on CNNs in many learning problems. In particular, we establish minimax optimal convergence rates of the least squares based on CNNs for learning smooth functions in the nonparametric regression setting. For binary classification, we derive convergence rates for CNN classifiers with hinge loss and logistic loss. It is also shown that the obtained rates for classification are minimax optimal in some common settings.
\end{abstract}



\section{Introduction}

Deep leaning has made remarkable successes in many applications and research fields such as image classification, speech recognition, natural language processing and scientific computing \cite{lecun2015deep,goodfellow2016deep}. This breakthrough of deep learning also motivated many theoretical researches on understanding and explaining the empirical successes of deep neural networks from various perspectives. In particular, recent studies have established optimal approximations of smooth function classes by fully connected neural networks \cite{yarotsky2017error,yarotsky2018optimal,shen2020deep,lu2021deep}. It has also been shown that these networks can achieve minimax optimal rates of convergence in many learning problems, including nonparametric regression \cite{schmidthieber2020nonparametric,kohler2021rate} and classification \cite{kim2021fast}.

Our main interests in this paper are the approximation and learning properties of convolutional neural networks (CNNs), which are widely used in image classification and related applications \cite{krizhevsky2012imagenet}. Recently, substantial progress has been made in the theoretical study of CNNs. It has been shown that CNNs are universal for approximation \cite{zhou2020universality} and universally consistent for nonparametric regression \cite{lin2022universal}. Approximation bounds and representational advantages of CNNs have been proven in several works \cite{oono2019approximation,zhou2020theory,fang2020theory,mao2021theory}. Furthermore, rates of convergence of estimators based on CNNs were established for nonparametric regression \cite{zhou2024learning,yang2024optimal} and classification \cite{kohler2020statistical,liu2021besov,feng2021generalization}. However, in contrast with the minimax optimal learning rates for fully connected neural networks \cite{schmidthieber2020nonparametric,kohler2021rate,kim2021fast}, many of these results for CNNs are not optimal.

In this paper, we take a step to close this gap by providing new analysis on the approximation and learning capacities of CNNs. Specifically, we prove new bounds for the approximation of smooth functions by CNNs with certain constraint on the weights. We also derive bounds for the covering numbers of these networks. Using the obtained bounds, we are able to establish convergence rates for CNNs in many learning problems. These rates are known to be minimax optimal in several settings. We summarize our contributions in the following.

\begin{enumerate}[label=\textnormal{(\arabic*)}]
\item We prove the rates $\cO(L^{-\alpha/d})$ for the approximation of smooth functions with smoothness $\alpha<(d+3)/2$ by CNNs, where $d$ is the dimension and $L$ is the depth of CNNs. The main advantage of our result is that we have an explicit control on the network weights (through $\kappa(\theta)$ defined by (\ref{norm constraint}) below). It is proven that one can choose $\kappa(\theta)\le M$ with $M= \cO(L^{\frac{3d+3-2\alpha}{2d}})$ to ensure that the approximation rate $\cO(L^{-\alpha/d})$ holds. 

\item We provide a new framework to estimate the covering numbers of feed-forward neural networks. An application of this result gives the bound $\cO(L\log(LM/\epsilon))$ for the $\epsilon$-covering number of CNNs with depth $L$ and weight constraint $\kappa(\theta)\le M$. When $M$ grows at most polynomially on $L$, our bound is better than the general bound $\cO(L^2 \log(L/\epsilon))$ in the literature.

\item For regression, we establish the minimax optimal rate for the least squares regression with CNNs, when the regression function is smooth.

\item For binary classification, we establish rates of convergence for CNN classifiers with hinge loss and logistic loss, under the Tsybakov noise condition (\ref{Tsybakov}). For the hinge loss, the obtained rate for the excess classification risk is minimax optimal. For the logistic loss, the obtained rate may not be optimal for the excess classification risk. But it is optimal for the excess logistic risk, at least in some situations.
\end{enumerate}

The remainder of this paper is organized as follows. In Section \ref{sec: cnn}, we describe the architecture of convolutional neural networks used in this paper, and derive bounds for the approximation capacity and covering number of these networks. Sections \ref{sec: regression} and \ref{sec: classification} study the nonparametric regression and classification problems, respectively. We give convergence rates of the excess risk for CNNs in these two sections. Section \ref{sec: conclusion} concludes this paper with a discussion on future studies. Omitted proofs are given in Supplementary Materials.

\subsection{Notations}\label{sec: notations}

For $i,j\in \bZ$ with $i\le j$, we use the notation $[i:j]:=\{i,i+1,\dots,j\}$. When $i=1$, we also denote $[j]:=[1:j]$ for convenience. We use the following conversion for tensors, where we take the tensor $\Bx =(x_{i,j,k})_{i\in [m],j\in [n],k\in [r]} \in \bR^{m\times n\times r}$ as an example. We use $\|\Bx\|_p$ to denote the $p$-norm of the tensor $\Bx$ by viewing it as a vector of $\bR^{mnr}$. The notation $x_{:,j,k}$ denotes the tensor $(x_{i,j,k})_{i\in [m]}\in \bR^m$, which is also viewed as a vector. We use $x_{:,j,:}$ to denote the tensor $(x_{i,j,k})_{i\in [m], k\in [r]}\in \bR^{m\times r}$. Other notations, such as $x_{i,:,k}$ and $x_{:,:,k}$, are similarly defined. If $X$ and $Y$ are two quantities, we denote their maximal value by $X\lor Y:= \max\{X,Y\}$. We use $X \lesssim Y$ or $Y \gtrsim X$ for two sequences $X,Y$ to denote the statement that $X\le CY$ for some constant $C>0$. We also denote $X \asymp Y$ when $X \lesssim Y \lesssim X$. The notation $X\lesssim \Poly(Y)$ means that $X$ is smaller than some polynomial of $Y$. For any function $f:\bR \to \bR$, we will often extend its definition to $\bR^d$ by applying $f$ coordinate-wisely. Throughout this paper, we assume that the dimension $d\ge 2$ is a fixed integer.

\section{Convolutional neural networks}\label{sec: cnn}

Let us first define convolutional neural networks used in this paper. Let $\Bw =(w_1,\dots,w_s)^\intercal \in \bR^s$ be a filter with filter size $s\in [d]$. We define the convolution matrix $T_\Bw$ on $\bR^d$ by
\[
T_\Bw := 
\begin{pmatrix}
w_1 &\cdots &w_{s-1} &w_s \\
 &\ddots &\ddots &\ddots &\ddots \\
 & & w_1 &\cdots &w_{s-1} &w_s \\
 & & & w_1 &\cdots &w_{s-1} \\
 & & & &\ddots &\vdots \\
 & & & & &w_1
\end{pmatrix}\in \bR^{d\times d}
\]
This convolution matrix corresponds to the one-sided padding and stride-one convolution by the filter $\Bw$. It is essentially the same as the convolution used in \cite{oono2019approximation} (up to a matrix transpose). But it is different from the convolution matrix used in \cite{zhou2020theory,zhou2020universality,fang2020theory,mao2021theory,feng2021generalization,yang2024optimal}, which is of dimension $(d+s)\times d$, rather than $d\times d$. So, in their setting, the network width increases after every application of the convolution with only one channel, while the network width remains the same with more channels in our setting. We define convolutional layers as follows. Let $s, J, J' \in \bN$ be a filter size, input channel size, and output channel size. For a filter $\Bw = (w_{i,j',j})_{i\in [s],j'\in [J'],j\in [J]} \in \bR^{s\times J'\times J}$ and a bias vector $\Bb = (b_1,\dots, b_{J'})^\intercal\in \bR^{J'}$, we define the convolutional layer as an operator $\Conv_{\Bw,\Bb}: \bR^{d\times J} \to \bR^{d\times J'}$ by
\[
(\Conv_{\Bw,\Bb} (\Bx))_{:,j'} := \sum_{j=1}^J T_{w_{:,j',j}} x_{:,j} + b_{j'},\quad \Bx=(x_{i,j})_{i\in [d],j\in [J]} \in \bR^{d\times J},
\]
where we use ``$+ b_{j'}$'' to denote the vector addition ``$+ (b_{j'},\dots,b_{j'})^\intercal$'' when there is no confusion. Next, we define convolutional neural networks. Let $s\in [d]$ and $J,L \in \bN$ be the filter size, channel size and depth. We denote by $\CNN(s,J,L)$ the set of functions $f_\Btheta$ that can be parameterized by $\Btheta=(\Bw^{(0)},\Bb^{(0)},\dots, \Bw^{(L-1)},\Bb^{(L-1)},\Bw^{(L)})$ in the following form
\begin{equation}\label{CNN}
f_\Btheta(\Bx) := \left\langle \Bw^{(L)}, \sigma \circ \Conv_{\Bw^{(L-1)},\Bb^{(L-1)}} \circ \cdots \circ \sigma \circ \Conv_{\Bw^{(0)},\Bb^{(0)}} (\Bx) \right\rangle,\quad \Bx\in [0,1]^d,
\end{equation}
where $\Bw^{(0)} \in \bR^{s\times J \times 1}, \Bb^{(0)}\in \bR^J, \Bw^{(L)} \in \bR^{d\times J}, \Bw^{(\ell)} \in \bR^{s\times J \times J}, \Bb^{(\ell)} \in \bR^{J}$ for $\ell \in [L-1]$, and the activation $\sigma(t) = t \lor 0$ is the ReLU activation function. Note that we have assumed the channel sizes in each layers are the same, because we can always increase the channel sizes by adding appropriate zero filters and biases. For convenience, we will often view $\Bw^{(0)} \in \bR^{s\times J \times J}$ by adding zeros to the filter and the input. The number of parameters in the network is $(sJ+1)JL +(d+s-sJ)J \lesssim J^2L$, which grows linearly on the depth $L$.

In order to control the complexity of convolutional neural networks, we introduce the following norm for the pair $(\Bw,\Bb) \in \bR^{s\times J \times J} \times \bR^J$
\[
\| (\Bw,\Bb) \| := \max_{j'\in [J]} \left( \left\| w_{:,j',:} \right\|_1 + |b_{j'}|\right).
\]
Note that $\| (\Bw,\Bb) \|$ quantifies the size of the affine transform $\Conv_{\Bw,\Bb}$:
\begin{align}
\left\|\Conv_{\Bw,\Bb} (\Bx)\right\|_\infty &\le \max_{j'\in [J]} \left(\sum_{j=1}^J \left\| T_{w_{:,j',j}} x_{:,j} \right\|_\infty + |b_{j'}|\right) \nonumber\\
&\le \| (\Bw,\Bb) \| (\|\Bx\|_\infty \lor 1).\label{norm of conv}
\end{align}
Following the idea of \cite{jiao2023approximation}, we define a constraint on the weights as follows
\begin{equation}\label{norm constraint}
\kappa(\Btheta) := \|\Bw^{(L)}\|_1 \prod_{\ell=0}^{L-1} \left(\| (\Bw^{(\ell)},\Bb^{(\ell)}) \| \lor 1 \right).
\end{equation}
For any $M\ge 0$, we denote the function class consisting of wights constrained CNNs by
\[
\CNN(s,J,L,M) := \left\{ f_\Btheta \in \CNN(s,J,L): \kappa(\Btheta) \le M \right\}.
\]
Several properties of this function class are summarized in Section \ref{sec: property} of Supplementary Materials. In Sections \ref{sec: app} and \ref{sec: cover unmber}, we study the approximation capacity and covering number of $\CNN(s,J,L,M)$. These results are used in Sections \ref{sec: regression} and \ref{sec: classification} to study the convergence rates of CNNs on the nonparametric regression and classification problems.

\subsection{Approximation}\label{sec: app}

We consider the capacity of CNNs for approximating smooth functions. Given a smoothness index $\alpha>0$, we write $\alpha=r+\beta$ where $r\in \bN_0 :=\bN \cup\{0\}$ and $\beta \in(0,1]$. Let $C^{r,\beta}(\bR^d)$ be the H\"older space with the norm
\[
\|f\|_{C^{r,\beta}(\bR^d)} := \max\left\{ \|f\|_{C^r(\bR^d)}, \max_{\|\Bs\|_1=r}|\partial^\Bs f|_{C^{0,\beta}(\bR^d)} \right\},
\]
where $\Bs=(s_1,\dots,s_d) \in \bN_0^d$ is a multi-index and 
\begin{align*}
\|f\|_{C^r(\bR^d)} &:= \max_{\|\Bs\|_1\le r} \|\partial^\Bs f\|_{L^\infty(\bR^d)}, \\
|f|_{C^{0,\beta}(\bR^d)} &:= \sup_{\Bx\neq \By\in \bR^d} \frac{|f(\Bx)-f(\By)|}{\|\Bx-\By\|_2^\beta}.
\end{align*}
Here, we use $\|\cdot\|_{L^\infty}$ to denote the supremum norm since we only consider continuous functions. We write $C^{r,\beta}([0,1]^d)$ for the Banach space of all
restrictions to $[0,1]^d$ of functions in $C^{r,\beta}(\bR^d)$. The norm is given by $\|f\|_{C^{r,\beta}([0,1]^d)} = \inf\{ \|g\|_{C^{r,\beta}(\bR^d)}: g\in C^{r,\beta}(\bR^d) \mbox{ and } g=f \mbox{ on } [0,1]^d\}$. For convenience, we will denote the ball of $C^{r,\beta}([0,1]^d)$ with radius $R>0$ by 
\[
\cH^\alpha(R):= \left\{ f\in C^{r,\beta}([0,1]^d): \|f\|_{C^{r,\beta}([0,1]^d)}\le R \right\}.
\]
Note that, for $\alpha=1$, $\cH^1(R)$ is a class of Lipschitz continuous functions.

Our first result estimates the error of approximating H\"older functions by CNNs.

\begin{theorem}\label{app bound}
Let $0<\alpha<(d+3)/2$ and $s\in [2:d]$. If $L\ge \lceil \frac{d-1}{s-1} \rceil$ and $M\gtrsim L^{\frac{3d+3-2\alpha}{2d}}$, then
\[
\sup_{h\in \cH^\alpha(1)} \inf_{f\in \CNN(s,6,L,M)} \|h-f\|_{L^\infty([0,1]^d)} \lesssim L^{-\frac{\alpha}{d}}.
\]
\end{theorem}

The approximation rate $\cO(L^{-\alpha/d})$ is slightly better than $\cO((L/\log L)^{-\alpha/d})$ in \cite[Corollary 4]{oono2019approximation} and \cite[Theorem 1]{liu2021besov} for ResNet-type CNNs. Furthermore, the results of \cite{oono2019approximation,liu2021besov} requires that the depth of residual blocks grows with the approximation error, while our result does not need any residual blocks. Our approximation rate is the same as the result of \cite{yang2024optimal}, which used slightly different CNNs, and the rate in \cite{feng2021generalization}, which considered the approximation of smooth functions on spheres. There is another recent paper \cite{shen2022approximation} proving the so-called super-convergence rate $\cO((L/\log L)^{-2\alpha/d})$ for CNNs by combining the super-convergence rate for fully-connected networks \cite{lu2021deep,jiao2023deep} and the result of \cite{zhou2020theory}, which showed that fully-connected networks can be implemented by CNNs. Note that the network architecture in \cite{shen2022approximation} is also different to ours because they need downsampling layers in order to apply \cite[Theorem 2]{zhou2020theory}. We summarize and compare the network architectures and approximation results of these papers in Table \ref{table}.

\begin{table}[htbp]
\footnotesize
\caption{A comparison of network architectures and approximation results for CNNs in recent works. The rate of approximation by CNNs with depth $L$ is given for target functions with smoothness $\alpha$ and input dimension $d$. ``Residual'', ``Downsampling'' and ``FC'' mean that residual blocks, downsampling layers and fully-connected layers are used, respectively.}\label{table}
\begin{center}
  \begin{tabular}{|c|c|c|c|c|} \hline
    & Network architecture & Weight constraint & Target function & Approximation rate  \\ \hline
    \cite{oono2019approximation} & CNN+Residual & maximum magnitude & H\"older  & $(L/\log L)^{-\alpha/d}$  \\ \hline
    \cite{liu2021besov} & CNN+Residual & maximum magnitude & Besov & $(L/\log L)^{-\alpha/d}$ \\ \hline 
	\cite{feng2021generalization} & CNN+Downsampling+FC & None & Sobolev on sphere  & $L^{-\alpha/d}$ \\ \hline 
	\cite{shen2022approximation} & CNN+Downsampling & None & Sobolev & $(L/\log L)^{-2\alpha/d}$\\ \hline 
	\cite{yang2024optimal} & CNN & None & H\"older, $\alpha<(d+3)/2$ & $L^{-\alpha/d}$\\ \hline 
	Ours & CNN & $\kappa(\theta)$ defined by (\ref{norm constraint}) & H\"older, $\alpha<(d+3)/2$ & $L^{-\alpha/d}$\\ \hline 
  \end{tabular}
\end{center}
\end{table}

Approximation results for neural networks can be divided into two categories according to whether the network weights are constrained. When there is no weight constraint, one can derive super-convergence rate by using the bit extraction technique \cite{bartlett2019nearly,lu2021deep,shen2022optimal} and estimate the complexity of the network using VC-dimension \cite{bartlett2019nearly,haussler1992decision,kohler2021rate}. However, if the magnitudes of the weights are constrained, it seems that one can only get the slow rate $\cO(L^{-\alpha/d})$. In this case, we can directly estimate the covering number of the network \cite{feng2021generalization,oono2019approximation,schmidthieber2020nonparametric}. Comparing with existing results, which often bound the maximum magnitude of the weights, the main advantage of Theorem \ref{app bound} is that we provide an explicitly bound on the weight constraint $\kappa(\theta) \le M$, which leads to an optimal estimate of the covering number as shown by Theorem \ref{cover num bound} below.

It is difficult to directly construct CNNs to approximate smooth functions. As mentioned above, most existing works derive approximation rates for CNNs by using the idea that smooth functions are well approximated by fully-connected neural networks and one can construct CNNs to implement fully-connected neural networks \cite{oono2019approximation,zhou2020universality,zhou2020theory,shen2022approximation}. Our proof of Theorem \ref{app bound} is also based on this idea. The main technical difference from previous works is that we apply the approximation bound for shallow neural networks proven in \cite{yang2024optimal}. To be concrete, let us denote the function class of shallow neural networks by
\begin{equation}\label{shallow nn}
\NN(N,M) := \left\{ f(\Bx) = \sum_{i=1}^N c_i\sigma(\Ba_i^\intercal \Bx +b_i): \sum_{i=1}^N |c_i|(\|\Ba_i\|_1 + |b_i|)\le M \right\}.
\end{equation}
It was shown by \cite[Corollary 2.4]{yang2024optimal} that, if $\alpha<(d+3)/2$, then
\begin{equation}\label{shallow nn bound}
\sup_{h\in \cH^\alpha(1)} \inf_{f\in \NN(N,M)} \|h-f\|_{L^\infty([0,1]^d)} \lesssim N^{-\frac{\alpha}{d}} \lor M^{-\frac{2\alpha}{d+3-2\alpha}}.
\end{equation}
In order to apply the bound (\ref{shallow nn bound}), we first construct a CNN to implement the function of the form $\Bx \mapsto c\sigma(\Ba^\intercal \Bx+b)$. Different from previous works on CNNs \cite{zhou2020universality,zhou2020theory}, we give an explicit estimate on the size of the weights.

\begin{lemma}\label{one neuron}
Let $s\in [2:d]$ and $L=\lceil \frac{d-1}{s-1} \rceil$. For any $\Ba \in \bR^d$ and $b, c\in \bR$, there exists $f\in \CNN(s,3,L,M)$ such that $f(\Bx) = c\sigma(\Ba^\intercal \Bx+b)$ for $\Bx\in [0,1]^d$ and $M= 3^{L-1} |c|(\|\Ba\|_1+|b|)$. Furthermore, the output weights $\Bw^{(L)} \in \bR^{d\times 3}$ can be chosen to satisfy $w^{(L)}_{i,j}= 0$ except for $i=j=1$.
\end{lemma}

Using this result, we further show that shallow neural networks defined by (\ref{shallow nn}) can be parameterized by a CNN in the following lemma. Theorem \ref{app bound} is a direct consequence of the approximation bound (\ref{shallow nn bound}) and this lemma. The detailed proofs are given in Supplementary Material.

\begin{lemma}\label{CNN construction}
Let $s\in [2:d]$ and $L_0=\lceil \frac{d-1}{s-1} \rceil$. Then, for any $f\in\NN(N,M)$, there exists $f_\Btheta \in \CNN(s,6,NL_0,3^{L_0+1} NM)$ such that $f_\Btheta(\Bx)=f(\Bx)$ for all $\Bx\in[0,1]^d$.
\end{lemma}

Note that it is possible to extend the approximation bound (\ref{shallow nn bound}) to Sobolev spaces with smoothness $\alpha\le (d+3)/2$ by using the Radon transform as done in the recent paper \cite{mao2024approximation}. Consequently, one can also generalize Theorem \ref{app bound} and hence the statistical learning bounds in Sections \ref{sec: regression} and \ref{sec: classification} to these spaces by using the same argument. The restriction on the smoothness $\alpha< (d+3)/2$ is of course due to the use of the approximation bound (\ref{shallow nn bound}). For high smoothness $\alpha>(d+3)/2$, one can also derive approximation bounds for CNNs by using the results of \cite{yang2024optimal}, as discussed in the following remark.

\begin{remark}\label{app remark}
If $\alpha>(d+3)/2$, it was shown by \cite[Theorem 2.1]{yang2024optimal} that $\cH^\alpha(1) \subseteq \cF_\sigma(M)$ for some constant $M>0$, where 
\begin{equation}\label{variation space}
\cF_\sigma(M):= \left\{ f_\mu(\Bx)=\int_{\bS^d} \sigma((\Bx^\intercal,1) \Bv) d\mu(\Bv): \|\mu\|\le M \right\}.
\end{equation}
Here, $\bS^d$ is the unit sphere of $\bR^{d+1}$ and $\|\mu\|=|\mu|(\bS^d)$ is the total variation of the measure $\mu$. $\cF_\sigma(M)$ is a ball with radius $M$ of the variation space corresponding to shallow ReLU neural networks studied in many recent papers, such as \cite{bach2017breaking,weinan2022barron,siegel2022sharp,siegel2023characterization,siegel2023optimal}. The function class $\cF_\sigma(M)$ can be viewed as output functions of an infinitely wide neural network. It is the limit of $\NN(N,M)$ as the number of neurons $N\to \infty$ \cite[Proposition 2.2]{yang2024optimal}. The recent work \cite{siegel2023optimal} showed that 
\[
\sup_{h\in \cF_\sigma(1)} \inf_{f\in \NN(N,1)} \|h-f\|_{L^\infty([0,1]^d)} \lesssim N^{-\frac{d+3}{2d}}.
\]
Combining this bound with Lemma \ref{CNN construction}, we can obtain
\begin{equation}\label{app bound for F_sigma}
\sup_{h\in \cF_\sigma(1)} \inf_{f\in \CNN(s,6,L,M)} \|h-f\|_{L^\infty([0,1]^d)} \lesssim L^{-\frac{d+3}{2d}},
\end{equation}
for $M\gtrsim L$. This approximation bound can be used to study machine learning problems with smoothness assumption $\alpha>(d+3)/2$, see Remark \ref{regression remark} for example. However, the bound (\ref{app bound for F_sigma}) is not optimal for CNNs and high smoothness. In order to get a better approximation rate, a possible way is to first derive a bound similar to (\ref{shallow nn bound}) for high smoothness and deep networks with proper constraint on the weights (see \cite{jiao2023approximation} for instance), and then show that these networks can be implemented by CNNs. We leave this for future study.
\end{remark}

\subsection{Covering number}\label{sec: cover unmber}

In statistical learning theory, we often estimate generalization error of learning algorithms by certain complexities of models. The complexity we use in this paper is the covering number (or metric entropy) defined in the following.

\begin{definition}[Covering number and entropy]
Let $\rho$ be a metric on a metric space $\cM$ and $\cF\subseteq \cM$. For $\epsilon>0$, a set $\cS \subseteq \cM$ is called an $\epsilon$-cover (or $\epsilon$-net) of $\cF$ if for any $x\in \cF$, there exists $y\in \cS$ such that $\rho(x,y)\le \epsilon$. The $\epsilon$-covering number of $\cF$ is defined by
\[
\cN(\epsilon,\cF,\rho) := \min\{|\cS|: \cS \mbox{ is an $\epsilon$-cover of } \cF \},
\]
where $|\cS|$ is the cardinality of the set $\cS$. The logarithm of the covering number $\log \cN(\epsilon,\cF,\rho)$ is called (metric) entropy.
\end{definition}

It is often the case that the metric $\rho$ is induced by a norm $\|\cdot\|$. In this case, we denote the $\epsilon$-covering number by $\cN(\epsilon,\cF,\|\cdot\|)$ for convenience. We will mostly consider the covering number of function classes $\cF$ parameterized by neural networks in the normed space $L^\infty([0,1]^d)$. In the following, we first give a general framework to estimate the covering numbers of feed-forward neural networks and then apply the result to CNNs.

We consider neural networks of the following form
\begin{equation}\label{general NN}
\begin{aligned}
\Bf_0(\Bx) &= \Bx \in [0,1]^d, \\
\Bf_{\ell+1}(\Bx) &= \sigma(\Bvarphi_{\Btheta_\ell}(\Bf_{\ell}(\Bx))), \quad \ell \in [0:L-1], \\
f_\Btheta(\Bx) &= \Bvarphi_{\Btheta_L}(\Bf_{L}(\Bx)),
\end{aligned}
\end{equation}
where $\Bvarphi_{\Btheta_\ell} : \bR^{d_\ell} \to \bR^{d_{\ell+1}}$ is an affine map parameterized by a vector $\Btheta_\ell \in \bR^{N_\ell}$ with $d_0 =d$, $d_{L+1}=1$ and the vector of parameters $\Btheta :=(\Btheta_0^\intercal,\dots,\Btheta_L^\intercal)^\intercal \in \bR^N$. Here, we use $N = \sum_{\ell=0}^L N_\ell$ to denote the number of parameters in the network. Note that we have restricted the input of the networks to $[0,1]^d$ for convenience. We assume that the parameterization satisfies the following conditions: for any $\Bx,\Bx'\in \bR^{d_\ell}$ and $\ell \in [0:L]$,
\begin{equation}\label{para assumption}
\begin{aligned}
\|\Btheta\|_\infty &\le B, \\
\|\Bvarphi_{\Btheta_\ell}(\Bx) \|_\infty &\le \gamma_\ell(\|\Bx\|_\infty \lor 1), \\
\|\Bvarphi_{\Btheta_\ell}(\Bx) - \Bvarphi_{\Btheta_\ell}(\Bx')\|_\infty &\le \gamma_\ell \|\Bx - \Bx'\|_\infty, \\
\|\Bvarphi_{\Btheta_\ell}(\Bx) - \Bvarphi_{\Btheta_\ell'}(\Bx)\|_\infty &\le \lambda_\ell \|\Btheta_\ell - \Btheta_\ell'\|_\infty (\|\Bx\|_\infty \lor 1),
\end{aligned}
\end{equation}
where $\Btheta'$ denotes any parameters satisfying $\|\Btheta'\|_\infty \le B$. 

Note that, if the affine map have the following matrix form
\[
\Bvarphi_{\Btheta_\ell}(\Bx) = \BA_{\Btheta_\ell} \Bx + \Bb_{\Btheta_\ell} = (\BA_{\Btheta_\ell}, \Bb_{\Btheta_\ell}) 
\begin{pmatrix}
\Bx \\
1
\end{pmatrix},
\]
then we can choose $\gamma_\ell$ to be the matrix operator norm (induced by $\|\cdot\|_\infty$)
\[
\gamma_\ell = \left\| (\BA_{\Btheta_\ell}, \Bb_{\Btheta_\ell}) \right\|_{l^\infty \to l^\infty},
\]
and choose $\lambda_\ell$ to be the Lipschitz constant of the parameterization
\[
\left\| (\BA_{\Btheta_\ell}, \Bb_{\Btheta_\ell}) - (\BA_{\Btheta_\ell'}, \Bb_{\Btheta_\ell'}) \right\|_{l^\infty \to l^\infty} \le \lambda_\ell \|\Btheta_\ell - \Btheta_\ell'\|_\infty.
\]
Recall that the matrix operator norm $\left\| (\BA, \Bb) \right\|_{l^\infty \to l^\infty}$ is the maximal $1$-norm of rows of the matrix $(\BA, \Bb)$. In our constructions, $(\BA_{\Btheta_\ell}, \Bb_{\Btheta_\ell})$ is linear on the parameter $\Btheta_\ell$, which implies that we can choose $\lambda_\ell \lesssim d_\ell+1$. 

The next lemma estimates the covering numbers of the neural networks described above.

\begin{lemma}\label{covering num}
Let $\cF$ be the class of functions $f_\Btheta$ that can be parameterized in the form (\ref{general NN}), where the parameterization satisfies (\ref{para assumption}) with $\lambda_\ell\ge 0$ and $\gamma_\ell\ge 1$ for $\ell =[0:L]$. Then, the $\epsilon$-covering number of $\cF$ in the $L^\infty([0,1]^d)$ norm satisfies
\[
\cN(\epsilon,\cF,\|\cdot\|_{L^\infty([0,1]^d)})\le (C_L B/\epsilon)^N,
\]
where $N$ is the number of parameters and $C_L$ can be computed inductively by
\[
C_0=\lambda_0,\quad C_{\ell+1}= \gamma_{\ell+1} C_\ell + \lambda_{\ell+1} \prod_{i=0}^{\ell} \gamma_i.
\]
In particular,
\[
C_L \le \left( \sum_{j=0}^L \lambda_j \right) \prod_{i=0}^L \gamma_i.
\]
\end{lemma}
\begin{proof}
For any $\Btheta,\Btheta'\in [-B,B]^N$ with $\|\Btheta-\Btheta'\|_\infty \le \epsilon$, we claim that, for any $\Bx\in [0,1]^d$ and $\ell \in [0:L]$, 
\begin{align*}
\|\Bf_\ell(\Bx)\|_\infty \le \prod_{i=-1}^{\ell-1} \gamma_i, \\
\|\Bvarphi_{\Btheta_\ell}(\Bf_\ell(\Bx)) - \Bvarphi_{\Btheta_\ell'}(\Bf_\ell'(\Bx))\|_\infty \le C_\ell \epsilon, \\
C_\ell \le \left( \sum_{j=0}^\ell \lambda_j \right) \prod_{i=0}^\ell \gamma_i,
\end{align*}
where we set $\gamma_{-1}=1$ and $\Bf_\ell'$ denotes the function in (\ref{general NN}) parameterized by $\Btheta'$. Thus, any $\epsilon$-cover of $[-B,B]^N$ gives a $C_L \epsilon$-cover of $\cF$ in the $L^\infty([0,1]^d)$ norm. Since the $\epsilon$-covering number of $[-B,B]^N$ is at most $(B/\epsilon)^N$, we get the desire bound for $\cN(\epsilon,\cF,\|\cdot\|_{L^\infty([0,1]^d)})$.

We prove the claim by induction on $\ell \in [0:L]$. The claim is trivial for $\ell=0$ by definition. Assume that the claim is true for some $0\le \ell<L$, we are going to prove it for $\ell+1$. By induction hypothesis,
\[
\|\Bf_{\ell+1}(\Bx)\|_\infty \le \|\Bvarphi_{\Btheta_\ell}(\Bf_\ell(\Bx))\|_\infty \le \gamma_\ell (\|\Bf_\ell(\Bx)\|_\infty \lor 1) \le \prod_{i=0}^{\ell} \gamma_i,
\]
where we used $\gamma_i\ge 1$ in the last inequality. By the Lipschitz continuity of ReLU, 
\[
\|\Bf_{\ell+1}(\Bx) - \Bf_{\ell+1}'(\Bx)\|_\infty \le \|\Bvarphi_{\Btheta_\ell}(\Bf_\ell(\Bx)) - \Bvarphi_{\Btheta_\ell'}(\Bf_\ell'(\Bx))\|_\infty \le C_\ell \epsilon.
\]
Therefore,
\begin{align*}
&\|\Bvarphi_{\Btheta_{\ell+1}}(\Bf_{\ell+1}(\Bx)) - \Bvarphi_{\Btheta_{\ell+1}'}(\Bf_{\ell+1}'(\Bx))\|_\infty \\
\le & \|\Bvarphi_{\Btheta_{\ell+1}}(\Bf_{\ell+1}(\Bx)) - \Bvarphi_{\Btheta_{\ell+1}'}(\Bf_{\ell+1}(\Bx))\|_\infty + \|\Bvarphi_{\Btheta_{\ell+1}'}(\Bf_{\ell+1}(\Bx)) - \Bvarphi_{\Btheta_{\ell+1}'}(\Bf_{\ell+1}'(\Bx))\|_\infty \\
\le & \lambda_{\ell+1} \epsilon (\|\Bf_{\ell+1}(\Bx)\|_\infty \lor 1) + \gamma_{\ell+1} \|\Bf_{\ell+1}(\Bx) - \Bf_{\ell+1}'(\Bx)\|_\infty \\
\le & \left(\lambda_{\ell+1} \prod_{i=0}^{\ell} \gamma_i + \gamma_{\ell+1} C_\ell\right) \epsilon = C_{\ell+1} \epsilon.
\end{align*}
Finally, by induction hypothesis and $\gamma_{\ell+1}\ge 1$,
\begin{align*}
C_{\ell+1} &= \gamma_{\ell+1} C_\ell + \lambda_{\ell+1} \prod_{i=0}^{\ell} \gamma_i \\
&\le \left( \sum_{j=0}^\ell \lambda_j \right) \prod_{i=0}^{\ell+1} \gamma_i + \lambda_{\ell+1} \prod_{i=0}^{\ell} \gamma_i \\
&\le \left( \sum_{j=0}^{\ell+1} \lambda_j \right) \prod_{i=0}^{\ell+1} \gamma_i,
\end{align*}
which completes the proof.
\end{proof}

Now, we apply Lemma \ref{covering num} to the convolutional neural network $\CNN(s,J,L,M)$. In this case, we have $\Bvarphi_{\Btheta_\ell} = \Conv_{\Bw^{(\ell)},\Bb^{(\ell)}}$ for $\ell \in [0:L-1]$ and $\Bvarphi_{\Btheta_L}(\cdot) = \langle \Bw^{(L)}, \cdot \rangle$. By Proposition \ref{rescaling} in Supplementary Materials, we can assume $\|\Bw^{(L)}\|_1\le M$ and $\|(\Bw^{(\ell)},\Bb^{(\ell)}) \|\le 1$ for all $\ell \in [0:L-1]$, which implies $B= M \lor 1$. Using the inequality (\ref{norm of conv}) and 
\begin{align*}
&\left\|\Conv_{\Bw,\Bb} (\Bx) - \Conv_{\Bw,\Bb} (\Bx') \right\|_\infty \\ 
\le& \max_{j'\in [J]} \left(\sum_{j=1}^J \left\| T_{w_{:,j',j}} x_{:,j} - T_{w_{:,j',j}} x_{:,j}' \right\|_\infty\right) \\
\le& \|(\Bw,\Bb)\| \|\Bx - \Bx'\|_\infty,
\end{align*}
we can set $\gamma_\ell = 1$ and $\gamma_L=M$. It is easy to see that we can choose $\lambda_\ell = sJ+1$ and $\lambda_L = dJ$. Consequently,
\[
C_L \le \left( \sum_{j=0}^L \lambda_j \right) \prod_{i=0}^L \gamma_i = (dJ+sJL+L)M \le 3dJLM,
\]
where we use $s\le d$ in the last inequality. We summarize the result in the next theorem for future reference.

\begin{theorem}\label{cover num bound}
Let $s,J,L\in \bN$ and $M\ge 1$. The entropy of $\CNN(s,J,L,M)$ satisfies
\[
\log \cN(\epsilon, \CNN(s,J,L,M),\|\cdot\|_{L^\infty([0,1]^d)}) \le N \log (3dJLM^2/\epsilon),
\]
where $N=(sJ+1)JL +(d+s-sJ)J$ is the number of parameters in the network.
\end{theorem}

In the analysis of neural networks, many papers, such as \cite{schmidthieber2020nonparametric,feng2021generalization}, simply assume that the parameters in the networks are bounded. In this case, the entropy is often bounded as $\cO(NL \log (N/\epsilon))$, where $N$ is the number of parameters and $L$ is the depth. For convolutional neural networks with bounded width, we have $N \asymp L$ and hence the entropy would be $\cO(L^2 \log (L/\epsilon))$. For comparison, Theorem \ref{cover num bound} gives a bound $\cO(L\log(LM/\epsilon))$. If one only assumes that the parameters are bounded by $B$, then $M\lesssim B^L$ and our bound is consistent with the previous bound. However, if the weight constraint $M$ grows at most polynomially on $L$, then we get a better bound $\cO(L \log (L/\epsilon))$ on the entropy. This improvement is essential to obtain optimal rates for many learning algorithms that we discuss in next two sections.

\section{Regression}\label{sec: regression}

In this section, we consider the classical nonparametric regression problem. Assume that $(\BX,Y)$ is a $[0,1]^d\times \bR$-valued random vector satisfying $\bE[Y^2]<\infty$. Let us denote the marginal distribution of $\BX$ by $\mu$ and the regression function by 
\[
h(\Bx) := \bE[Y|\BX=\Bx].
\] 
Suppose we are given a data set of $n$ samples $\cD_n = \{(\BX_i,Y_i)\}_{i=1}^n$, which are independent and have the same distribution as the random vector $(\BX,Y)$. The goal of nonparametric regression problem is to construct an estimator $\widehat{f}_n$, based on $\cD_n$, to reconstruct the regression function $h$. The estimation performance is evaluated by the $L^2$-error
\[
\|\widehat{f}_n - h\|_{L^2(\mu)}^2 = \bE_\BX \left[(\widehat{f}_n(\BX)-h(\BX))^2\right].
\] 

One of the popular algorithms to solve the regression problem is the empirical least squares
\begin{equation}\label{least squares}
\widehat{f}_n \in \argmin_{f\in \cF_n} \frac{1}{n} \sum_{i=1}^n (f(\BX_i)- Y_i)^2,
\end{equation}
where $\cF_n$ is a prescribed hypothesis class. For simplicity, we assume here and in the sequel that the minimum above indeed exists. We are interested in the case that the function class $\cF_n$ is parameterized by a CNN. In order to study the convergence rate of $\widehat{f}_n \to h$ as $n\to \infty$, we will assume that $h\in \cH^\alpha(R)$ for some constant $R>0$ and make the following assumption on the distribution of $(\BX,Y)$: there exists a constant $c>0$ such that
\begin{equation}\label{noise assumption}
\bE \left[ \exp(cY^2) \right] <\infty.
\end{equation}

In statistical analysis of learning algorithms, we often require that the hypothesis class is uniformly bounded. We define the truncation operator $\pi_B$ with level $B>0$ for real-valued functions $f$ as
\begin{equation}\label{truncation}
\pi_B f(\Bx) = 
\begin{cases}
B  \quad &f(\Bx)>B, \\
f(\Bx) \quad &|f(\Bx)| \le B, \\
-B \quad &f(\Bx)<-B.
\end{cases}
\end{equation}
Note that the truncation operator can be implemented by a CNN (see Lemma \ref{cnn composition} for example). Since we assume that the regression function $h$ is bounded, truncating the output of the estimator $\widehat{f}_n$ appropriately dose not increase the estimation error. The following theorem provides convergence rates for least squares estimators based on CNNs.

\begin{theorem}\label{rate regression}
Assume that the condition (\ref{noise assumption}) holds and the regression function $h\in\cH^\alpha(R)$ for some $0<\alpha<(d+3)/2$ and $R>0$. Let $\widehat{f}_n$ be the estimator defined by (\ref{least squares}) with $\cF_n = \CNN(s,J,L_n,M_n)$, where $s\in [2:d]$, $J\ge 6$ and
\[
L_n \asymp \left( \frac{n}{\log^3 n} \right)^{\frac{d}{2\alpha+d}}, \quad \left( \frac{n}{\log^3 n} \right)^{\frac{3d+3-2\alpha}{4\alpha+2d}} \lesssim M_n \lesssim \Poly(n).
\]
If $B_n = c_1\log n$ for some constant $c_1>0$, then
\[
\bE_{\cD_n} \left[\|\pi_{B_n}\widehat{f}_n-h\|_{L^2(\mu)}^2\right] \lesssim \left(\frac{\log^3 n}{n} \right)^{\frac{2\alpha}{2\alpha+d}}.
\]
\end{theorem}

It is well-known that the rate $n^{-\frac{2\alpha}{2\alpha+d}}$ is minimax optimal for learning functions in $\cH^\alpha(R)$ \cite{stone1982optimal}:
\[
\inf_{\widehat{f}_n} \sup_{h\in \cH^\alpha(R)} \bE_{\cD_n} \left[\|\widehat{f}_n-h\|_{L^2(\mu)}^2\right] \gtrsim n^{-\frac{2\alpha}{2\alpha+d}},
\]
where the infimum is taken over all estimators based on the training data $\cD_n$. Recent works have established the minimax rates (up to logarithm factors) for least squares estimators using fully-connected neural networks \cite{schmidthieber2020nonparametric,kohler2021rate,yang2024nonparametric}. For convolutional neural networks, \cite{oono2019approximation} proved the optimal rates for ResNet-type CNNs, under the requirement that the depth of the residual blocks grows with the sample size $n$, or the residual blocks are suitably masked. Theorem \ref{rate regression} removes the requirements on the residual blocks for low smoothness $\alpha<(d+3)/2$.

\begin{remark}\label{regression remark}
As we noted in Remark \ref{app remark}, if $\alpha>(d+3)/2$, then $\cH^\alpha(1) \subseteq \cF_\sigma(R)$ for some constant $R>0$. When the regression function $h\in \cF_\sigma(R)$, we can use the approximation bound (\ref{app bound for F_sigma}) to show that
\[
\bE_{\cD_n} \left[\|\pi_{B_n}\widehat{f}_n-h\|_{L^2(\mu)}^2\right] \lesssim \left(\frac{\log^3 n}{n} \right)^{\frac{d+3}{2d+3}},
\]
if we choose $L_n \asymp (n/\log^3 n)^{d/(2d+3)} \lesssim M_n \lesssim \Poly(n)$. This rate is minimax optimal (up to logarithm factors) for the function class $\cF_\sigma(R)$ \cite{yang2024optimal}. For comparison, \cite{yang2024optimal} only established the sub-optimal rate $\cO(n^{-\frac{d+3}{3d+3}}\log^4 n)$ for CNNs. Our result is also better than the recent analysis of CNNs in \cite{zhou2024learning}, which proved the rate $\cO(n^{-1/3}\log^2 n)$ for $\cH^\alpha(R)$ with $\alpha>(d+4)/2$.
\end{remark}

\section{Binary classification} \label{sec: classification}

In binary classification, we observe a dataset $\cD_n:= \{(\BX_i,Y_i): i=1,\dots,n\}$ of $n$ i.i.d. copies of a random vector $(\BX,Y)$, where we assume that the input vector $\BX \in [0,1]^d$ and the label $Y\in \{-1,1\}$. The marginal distribution of $\BX$ is denoted by $\bP_\BX$ and the conditional class probability function is denoted by
\[
\eta(\Bx) := \bP(Y=1|\BX=\Bx).
\]

For any real-valued function $f$ defined on $[0,1]^d$, we can define a classifier $\cC_f(\Bx):= \sgn(f(\Bx))$. The classification error of $f$ is defined as 
\[
\cE(f) = \bE_{\BX,Y}[\cC_f(\BX)\neq Y ] = \bE_{\BX,Y}[\boldsymbol{1}(Yf(\BX)<0)],
\] 
where $\boldsymbol{1}(\cdot)$ is $1$ if $(\cdot)$ is true, and is $0$ otherwise. A Bayes classifier $\cC^*=\cC_{f^*}$ is a classifier that minimizes the classification error $\cE(f^*)=\min_{f\in \cM} \cE(f)$, where $\cM$ is the set of all measurable functions on $[0,1]^d$. Note that $\cC^* = \sgn(2\eta -1)$ is a Bayes classifier and $\cE(\cC^*) = \frac{1}{2}\bE[1-|2\eta-1|]$. The goal of binary classification is to construct a classifier with small classification error by using the dataset $\cD_n$.

Since we only have finite observed samples, one natural approach to estimate the Bayes classifier is the empirical risk minimization (with $0-1$ loss)
\begin{equation}\label{0-1 loss}
\argmin_{f\in \cF_n} \frac{1}{n}\sum_{i=1}^n \boldsymbol{1}(Y_if(\BX_i)<0),
\end{equation}
where $\cF_n$ is a prescribed function class. However, this procedure is often computational infeasible due to the NP-hardness of the minimization problem. In general, one replaces the $0-1$ loss by surrogate losses. For a given surrogate loss function $\phi:\bR \to [0,\infty)$, the $\phi$-risk is defined as
\begin{align*}
\cL_{\phi}(f) &:= \bE_{\BX,Y}[\phi(Yf(\BX))] \\
&= \bE_\BX[\eta(\BX)\phi(f(\BX)) + (1-\eta(\BX))\phi(-f(\BX)) ].
\end{align*}
Its minimizer is denoted by $f^*_\phi \in \argmin_{f\in \cM} \cL_{\phi}(f)$. Note that $f^*_\phi$ can be explicitly computed by using the conditional class probability function $\eta$ for many convex loss functions $\phi$ \cite{zhang2004statistical,wu2007multi}. 
Instead of using (\ref{0-1 loss}), we can estimate the Bayes classifier by minimizing the empirical $\phi$-risk over a function class $\cF_n$:
\begin{equation}\label{ERM}
\widehat{f}_{\phi,n} \in \argmin_{f\in \cF_n} \frac{1}{n} \sum_{i=1}^n \phi(Y_i f(\BX_i)).
\end{equation}

The goal of this section is to estimate the convergence rates of the excess classification risk and excess $\phi$-risk defined by
\begin{align*}
\cR(\widehat{f}_{\phi,n}) &:= \cE(\widehat{f}_{\phi,n}) - \cE(\cC^*),\\
\cR_\phi(\widehat{f}_{\phi,n}) &:= \cL_{\phi}(\widehat{f}_{\phi,n}) - \cL_{\phi}(f^*_\phi),
\end{align*}
when $\cF_n$ is parameterized by a CNN. The convergence rates certainly depend on properties of the conditional class probability function $\eta$. One of the well known assumptions on $\eta$ is the Tsybakov noise condition \cite{mammen1999smooth,tsybakov2004optimal}: there exist $q\in [0,\infty]$ and $c_q>0$ such that for any $t>0$,
\begin{equation}\label{Tsybakov}
\bP_\BX(|2\eta(\BX)-1|\le t) \le c_q t^q.
\end{equation}
The constant $q$ is usually called the noise exponent. It is obvious that the Tsybakov noise condition always holds for $q=0$, whereas noise exponent $q=\infty$ means that $\eta$ is bounded away from the critical level $1/2$. We will consider classifications with hinge loss and logistic loss under the Tsybakov noise condition.

\subsection{Hinge loss}

For the hinge loss $\phi(t) = \max\{1-t,0\}$, we have $f^*_\phi = \sgn(2\eta-1) =\cC^*$ and $\cL_{\phi}(f^*_\phi) = \bE[1-|2\eta-1|]$. It is well known that the following calibration inequality holds \cite{zhang2004statistical,bartlett2006convexity}
\begin{equation}\label{hinge calibration}
\cR(f) \le \cR_\phi(f).
\end{equation}
Hence, any convergence rate for the excess $\phi$-risk $\cR_{\phi}(\widehat{f}_{\phi,n})$ implies the same convergence rate for the excess classification risk $\cR(\widehat{f}_{\phi,n})$. One can also check that \cite[Section 3.3]{zhang2004statistical}, if $|f|\le 1$, then
\begin{equation}\label{hinge caculation}
\cR_\phi(f) = \bE[|f-f^*_\phi||2\eta-1|].
\end{equation}
To use this equality, it is natural to truncate the output of the estimator by using the truncation operator $\pi_1$ defined by (\ref{truncation}). 

In the following theorem, we provide convergence rates for the excess $\phi$-risk of the CNN classifier with hinge loss, under the assumption that the conditional class probability function $\eta$ is smooth and satisfies the Tsybakov noise condition.

\begin{theorem}\label{rate hinge}
Assume the noise condition (\ref{Tsybakov}) holds for some $q\in [0,\infty]$ and $\eta\in \cH^\alpha(R)$ for some $0<\alpha<(d+3)/2$ and $R>0$. Let $\phi$ be the hinge loss and $\widehat{f}_{\phi,n}$ be the estimator defined by (\ref{ERM}) with $\cF_n = \{ \pi_1 f: f\in \CNN(s,J,L_n,M_n)\}$, where $s\in [2:d]$, $J\ge 6$ and
\[
L_n \asymp \left( \frac{n}{\log^2 n} \right)^{\frac{d}{(q+2)\alpha+d}}, \quad \left( \frac{n}{\log^2 n} \right)^{\frac{3d+3}{2(q+2)\alpha+2d}} \lesssim M_n \lesssim \Poly(n),
\]
then, for sufficiently large $n$,
\[
\bE_{\cD_n} \left[\cR_{\phi}(\widehat{f}_{\phi,n}) \right] \lesssim \left( \frac{\log^2 n}{n} \right)^{\frac{(q+1)\alpha}{(q+2)\alpha+d}}.
\]
\end{theorem}

It was shown in \cite{audibert2007fast} that the minimax lower bound for the excess classification risk is 
\begin{equation}\label{optimal rate}
\inf_{\widehat{f}_n} \sup_{\eta} \bE_{\cD_n} \left[\cR(\widehat{f}_n)\right] \gtrsim n^{-\frac{(q+1)\alpha}{(q+2)\alpha+d}},
\end{equation}
where the supremum is taken over all $\eta\in \cH^\alpha(R)$ that satisfies Tsybakov noise condition (\ref{Tsybakov}) and the infimum is taken over all estimators based on the training data $\cD_n$.
Hence, by the calibration inequality (\ref{hinge calibration}), the convergence rate in Theorem \ref{rate hinge} is minimax optimal up to a logarithmic factor. Similar results have been established in \cite{kim2021fast} for fully connected neural networks with hinge loss. However, their results rely on the sparsity of neural networks and hence one need to optimize over different network architectures to obtain the optimal rate, which is hard to implement due to the unknown locations of the non-zero parameters. Our result show that CNNs, whose architecture is specifically defined, are able to achieve the optimal rate.

\subsection{Logistic loss}

For the logistic loss $\phi(t)=\log(1+e^{-t})$, we have $f^*_\phi = \log(\frac{\eta}{1-\eta})$ and $\cL_{\phi}(f^*_\phi) = \bE[-\eta \log \eta - (1-\eta) \log(1-\eta)]$. Consequently, one can show that
\[
\cR_{\phi}(f) = \bE \left[\eta \log\left(\eta(1+e^{-f})\right) + (1-\eta)\log\left((1-\eta)(1+e^f)\right) \right].
\]
Let us denote the KL-divergence by
\[
\cD_{KL}(p_1,p_2):= p_1\log \left(\frac{p_1}{p_2}\right) + (1-p_1)\log \left(\frac{1-p_1}{1-p_2}\right),\quad p_1,p_2\in[0,1],
\] 
where $\cD_{KL}(p_1,p_2) =\infty$ if $p_2=0$ and $p_1\neq 0$, or $p_2=1$ and $p_1\neq 1$. If we define the logistic function by
\begin{equation}\label{logistic}
\psi(t) := \frac{1}{1+e^{-t}}\in [0,1],\quad t\in[-\infty,\infty],
\end{equation}
then a direct calculation shows that $\eta = \psi(f_\phi^*)$ and
\begin{equation}\label{logistic caculation}
\cR_{\phi}(f) = \bE[\cD_{KL}(\eta,\psi(f))].
\end{equation}
When the Tsybakov noise condition (\ref{Tsybakov}) holds, we have the following calibration inequality \cite[Theorem 8.29]{steinwart2008support}
\begin{equation}\label{logistic calibration}
\cR(f) \le 4c_q^{\frac{1}{q+2}} \cR_{\phi}(f)^{\frac{q+1}{q+2}}.
\end{equation}

For the logistic loss, the convergence rate depends not only on the Tsybakov noise condition, but also upon the Small Value Bound (SVB) condition introduced by \cite{bos2022convergence}. We say the distribution of $(\BX,Y)$ satisfies the SVB condition, if there exists $\beta\ge 0$ and $C_\beta>0$ such that for any $t\in(0,1]$,
\begin{equation}\label{SVB}
\bP_\BX (\eta(\BX) \le t)\le C_\beta t^\beta,\quad \bP_\BX (1-\eta(\BX) \le t)\le C_\beta t^\beta.
\end{equation}
Note that this condition always holds for $\beta=0$ with $C_\beta=1$. The index $\beta$ is completely determined by the behavior of $\eta$ near $0$ and $1$. If $\eta$ is bounded away form $0$ and $1$, then the SVB condition holds for all $\beta>0$. In contrast, the Tsybakov noise condition provides a control on the behavior of $\eta$ near the decision boundary $\{\Bx:\eta(\Bx)=1/2\}$. This difference is due to the loss: the $0-1$ loss only cares about the classification error, while the logistic loss measures how well the conditional class probability is estimated in the KL-divergence (\ref{logistic caculation}), which puts additional emphasis on small and large conditional class probabilities. 

The following theorem gives convergence rates for CNNs under the SVB condition. As pointed out by \cite{bos2022convergence}, we do not get any gain in the convergence rate when the SVB index $\beta>1$. So, we assume $\beta \in [0,1]$ in the theorem.

\begin{theorem}\label{rate logistic}
Assume the SVB condition (\ref{SVB}) holds for some $\beta\in[0,1]$ and $\eta\in \cH^\alpha(R)$ for some $0<\alpha<(d+3)/2$ and $R>0$. Let $\phi$ be the logistic loss and $\widehat{f}_{\phi,n}$ be the estimator defined by (\ref{ERM}) with $\cF_n = \{ \pi_{B_n} f: f\in \CNN(s,J,L_n,M_n)\}$, where $s\in [2:d]$, $J\ge 6$ and
\[
L_n \asymp \left( \frac{n}{\log n} \right)^{\frac{d}{(1+\beta)\alpha+d}}, \quad \left( \frac{n}{\log n} \right)^{\frac{3d+3+2\alpha}{2(1+\beta)\alpha+2d}} \lesssim M_n \lesssim \Poly(n), \quad B_n \asymp \log n,
\]
then, for sufficiently large $n$,
\[
\bE_{\cD_n} \left[\cR_{\phi}(\widehat{f}_{\phi,n}) \right] \lesssim \left( \frac{\log n}{n} \right)^{\frac{(1+\beta)\alpha}{(1+\beta)\alpha+d}} \log n.
\]
\end{theorem}

The convergence rate in Theorem \ref{rate logistic} is the same as \cite[Theorem 3.3]{bos2022convergence}, which studied multi-class classification using fully-connected deep neural networks with cross entropy loss. If, in addition, the Tsybakov noise condition (\ref{Tsybakov}) holds, by combining Theorem \ref{rate logistic} with the calibration inequality (\ref{logistic calibration}), we can get the following convergence rate for the classification risk:
\[
\bE_{\cD_n} [\cR(\widehat{f}_{\phi,n}) ] \lesssim n^{-\frac{q+1}{q+2} \frac{(1+\beta)\alpha}{(1+\beta)\alpha+d}} \log^2 n.
\]
Note that this rate is the same as the optimal rate (\ref{optimal rate}) up to a logarithmic factor, when $q=0$ and $\beta=1$. For $\beta=0$, the obtained rate is not minimax optimal for the excess classification risk. However, as shown by \cite[Corollary 2.1]{zhang2024classification}, the rate in Theorem \ref{rate logistic} is indeed minimax optimal up to a logarithmic factor for the excess $\phi$-risk when $\beta=0$. So, even if the logistic classification can achieve the minimax optimal convergence rates for classification, it is in general not possible to derive it through the rates for excess $\phi$-risk. 

There are other papers \cite{kohler2020statistical,liu2021besov,shen2022approximation} studying the convergence rates of CNNs with logistic loss. The paper \cite{kohler2020statistical} imposed a max-pooling structure for the conditional class probability that is related to the structure of convolutional networks. So, their result is not comparable to ours. \cite{liu2021besov} used a similar setting as ours and derived the rate $n^{-\frac{\alpha}{2\alpha+2(\alpha\lor d)}}$ (ignoring logarithmic factors) for the excess $\phi$-risk under the assumption that $\eta$ is supported on a manifold of $d$ dimension. The article \cite{shen2022approximation} also considered low-dimensional distributions and obtained the rate $n^{-\frac{\alpha}{2\alpha+d}}$ for the excess $\phi$-risk when $f^*_\phi\in \cH^\alpha$, which is more restricted than the assumption $\eta\in \cH^{\alpha}$. Our learning rate $n^{-\frac{\alpha}{\alpha+d}}$ in Theorem \ref{rate logistic} (for $\beta=0$) is better than those in \cite{liu2021besov} and \cite{shen2022approximation}, but their results can be applied to low-dimensional distributions. It would be interesting to generalize our result to these distributions.

\section{Conclusion}\label{sec: conclusion}

In this paper, we have studied approximation and learning capacities of convolutional neural networks with one-side zero-padding and multiple channels. We have derived new approximation bounds for CNNs with norm constraint on the weights. To study the generalization performance of learning algorithms induced by these networks, we also proved new bounds for their covering number. Based on these results, we established rates of convergence for CNNs in nonparametric regression and classification problems. Many of the obtained convergence rates are known to be minimax optimal.

There is a restriction on the smoothness of the target functions in our results. We think this restriction is due to the proof techniques of our approximation bound (Theorem \ref{app bound}), rather than the architecture of CNNs. It may be possible to use the ideas of network constructions from related works, such as \cite{oono2019approximation,zhou2020theory}, to remove the restriction, which we leave as a future work.

\section*{Acknowledgments}
This paper was submitted while Y. Yang was with City University of Hong Kong. The work described in this paper was partially supported by Discovery Project (DP240101919) of the Australian Research Council, InnoHK initiative, The Government of the HKSAR, Laboratory for AI-Powered Financial Technologies, the Research Grants Council of Hong Kong [Projects Nos. CityU 11315522, CityU 11303821] and National Natural Science Foundation of China [Project No. 12371103]. We thank the referees for their helpful comments and suggestions on the paper.

\bibliographystyle{siamplain}
\bibliography{references}

\end{document}


\maketitle

\section{Basic properties of CNNs}\label{sec: property}

In this section, we give some properties of the function class $\CNN(s,J,L,M)$, which will be useful for neural network construction.

\begin{proposition}\label{inclusion}
If $J\le J'$, $L\le L'$ and $M\le M'$, then it holds $\CNN(s,J,L,M) \subseteq \CNN(s,J',L',M')$.
\end{proposition}
\begin{proof}
It is easy to check that $\CNN(s,J,L,M) \subseteq \CNN(s,J',L,M')$ by adding zero filters appropriately and using the definition of weight constraint (\ref{norm constraint}). To prove the inclusion $\CNN(s,J',L,M') \subseteq \CNN(s,J',L',M')$, we observe that the convolution with the filter $(1,0,\dots,0)^\intercal\in \bR^s$ is the identity map of $\bR^d \to \bR^d$. Hence, we can increase the depth of CNN by adding the identity map in the last layer.
\end{proof}

The next proposition shows that we can always rescale the parameters in CNNs so that the norms of filters in the hidden layers are at most one.

\begin{proposition}[Rescaling]\label{rescaling}
Every $f\in \CNN(s,J,L,M)$ can be parameterized in the form (\ref{CNN}) such that $\|\Bw^{(L)}\|_1\le M$ and $\|(\Bw^{(\ell)},\Bb^{(\ell)}) \|\le 1$ for all $\ell \in [0:L-1]$.
\end{proposition}
\begin{proof}
The proof is essentially the same as \cite[Proposition 2.4]{jiao2023approximation}. We give the proof here for completeness. Note that $f\in \CNN(s,J,L,M)$ parameterized in the form (\ref{CNN}) can be written inductively by
\[
f(\Bx) = \langle \Bw^{(L)}, \Bf_{L}(\Bx) \rangle, \quad \Bf_{\ell+1}(\Bx) = \sigma \left(\Conv_{\Bw^{(\ell)},\Bb^{(\ell)}}(\Bf_{\ell}(\Bx))\right),\quad \Bf_0(\Bx) = \Bx.
\]
Let us denote $m_\ell:=\max \{\|(\Bw^{(\ell)},\Bb^{(\ell)}) \|,1\}$ for all $\ell\in [0:L-1]$ and let $\widetilde{\Bw}_{(\ell)} = \Bw_\ell/m_\ell$, $\widetilde{\Bb}_\ell = \Bb_\ell/(\prod_{i=0}^{\ell}m_i)$ and $\widetilde{\Bw}_{(L)}=\Bw^{(L)} \prod_{i=0}^{L-1}m_i$. We consider the functions defined inductively by
\[
\widetilde{\Bf}_{\ell+1}(\Bx) = \sigma \left(\Conv_{\widetilde{\Bw}^{(\ell)},\widetilde{\Bb}^{(\ell)}}(\widetilde{\Bf}_{\ell}(\Bx))\right),\quad \widetilde{\Bf}_0(\Bx) = \Bx.
\]
It is easy to check that $\|\widetilde{\Bw}_{(L)}\|\le M$ and
\[
\left\| (\widetilde{\Bw}^{(\ell)},\widetilde{\Bb}^{(\ell)}) \right\| = \frac{1}{m_\ell} \left\| \left(\Bw^{(\ell)}, \frac{\Bb^{(\ell)}}{\prod_{i=0}^{\ell-1}m_i} \right) \right\| \le \frac{1}{m_\ell} \left\|(\Bw^{(\ell)},\Bb^{(\ell)}) \right\| \le 1,
\]
where the first inequality is due to $m_i\ge 1$.

Next, we show that $\Bf_\ell(\Bx) = \left( \prod_{i=0}^{\ell-1} m_i \right) \widetilde{\Bf}_{\ell}(\Bx)$ by induction. For $\ell =1$, by the absolute homogeneity of the ReLU function,
\begin{align*}
\Bf_1(\Bx) &= \sigma \left(\Conv_{\Bw^{(0)},\Bb^{(0)}}(\Bx)\right) \\
&= m_0 \sigma \left(\Conv_{\widetilde{\Bw}^{(0)},\widetilde{\Bb}^{(0)}}(\Bx)\right) = m_0 \widetilde{\Bf}_1(\Bx).
\end{align*}
Inductively, one can conclude that
\begin{align*}
\Bf_{\ell+1}(\Bx) &= \sigma \left(\Conv_{\Bw^{(\ell)},\Bb^{(\ell)}}(\Bf_{\ell}(\Bx))\right) \\
&= \left( \prod_{i=0}^\ell m_i \right) \sigma \left(\Conv_{\widetilde{\Bw}^{(\ell)},\widetilde{\Bb}^{(\ell)}}\left( \frac{\Bf_\ell(\Bx)}{\prod_{i=0}^{\ell-1} m_i} \right)\right) \\
&= \left( \prod_{i=0}^\ell m_i \right) \sigma \left(\Conv_{\widetilde{\Bw}^{(\ell)},\widetilde{\Bb}^{(\ell)}}\left( \widetilde{\Bf}_{\ell}(\Bx) \right)\right) \\
&= \left( \prod_{i=0}^\ell m_i \right) \widetilde{\Bf}_{\ell+1}(\Bx),
\end{align*}
where the third equality is due to induction. Therefore,
\[
f(\Bx) = \langle \Bw^{(L)}, \Bf_{L}(\Bx) \rangle = \left\langle \Bw^{(L)}, \left( \prod_{i=0}^{L-1} m_i \right) \widetilde{\Bf}_{L}(\Bx) \right\rangle = \langle \widetilde{\Bw}^{(L)}, \widetilde{\Bf}_{L}(\Bx) \rangle,
\]
which means $f$ can be parameterized by $(\widetilde{\Bw}^{(0)},\widetilde{\Bb}^{(0)},\dots, \widetilde{\Bw}^{(L-1)},\widetilde{\Bb}^{(L-1)},\widetilde{\Bw}^{(L)})$ and we finish the proof.
\end{proof}

\section{A useful inequality}

\begin{proposition}\label{log2 bound}
If $p\in [0,1]$ and $q \in [u,1]$ with $0<u\le e^{-2}$, then
\[
p \log^2 (p/q) \le \log (u^{-2}) ( p \log (p/q)-p+q ).
\]
\end{proposition}
\begin{proof}
It is easy to check that the inequality holds for $p=0$, since $0\log^2 0=0\log 0=0$. So, we only consider $p>0$. If we denote $t=q/p\ge u$, then the desired inequality is equivalent to 
\begin{equation}\label{temp ineq}
t - 1 - \log t \ge \frac{\log^2 t}{-2\log u}. 
\end{equation}
It is easy to see that this inequality holds for $t=1$. For $t\in (1,\infty)$, it can be proven by letting $t=e^s$ with $s>0$ and using
\[
e^s - 1 - s \ge \frac{s^2}{2} \ge \frac{s^2}{-2\log u}.
\]
For $t\in [u,1)$, we consider the function 
\[
f(t) = \frac{\log^2 t}{t - 1 - \log t}.
\]
Since $t - 1 - \log t >0$, the inequality (\ref{temp ineq}) is equivalent to $f(t)\le -2\log u$. A direct calculation shows 
\begin{align*}
f'(t) &= \frac{\log t}{(t - 1 - \log t)^2} \left( 2 - \frac{2}{t} - \log t - \frac{\log t}{t} \right) \\
&=: \frac{\log t}{(t - 1 - \log t)^2} g(t).
\end{align*}
Since
\[
g'(t) = \frac{1-t+\log t}{t^2}<0,
\]
we know that $g(t)>g(1)=0$ and hence $f'(t)<0$ for $t\in [u,1)$. Therefore, $f$ is decreasing on $[u,1)$ and 
\[
f(t) \le f(u) = \frac{\log^2 u}{u - 1 - \log u} \le -2\log u,
\]
where we use $u-1\ge \log (e^{-1}) \ge \frac{1}{2} \log u$ because $u\le e^{-2}$. 
\end{proof}

\section{Proofs for Section \ref{sec: cnn}}\label{sec: proofs for sec cnn}

This section give the proofs of Theorem \ref{app bound}, Lemma \ref{one neuron} and Lemma \ref{CNN construction}.

\subsection{Proof of Lemma \ref{one neuron}}

By the homogeneity of ReLU, we can assume that $\|\Ba\|_1+|b|=1$. Observe that convolution with the filter $(u_1,\dots,u_s)^\intercal\in \bR^s$ computes the inner product with the first $s$ elements of the input signal $\sum_{i\in [s]} u_i x_i$. Convolution with the filter $(0,\dots,0,1)^\intercal\in \bR^s$ is the ``left-translation'' by $s-1$. We construct the CNN parameterized in the form (\ref{CNN}) as follows.

We define $\Bw^{(0)} \in \bR^{s\times 3\times 1}$ and $\Bb^{(0)}\in \bR^3$ by 
\[
w^{(0)}_{:,1,1} = - w^{(0)}_{:,2,1} =
\begin{pmatrix}
a_1 \\
\vdots \\
a_s
\end{pmatrix},\quad
w^{(0)}_{:,3,1} =
\begin{pmatrix}
0 \\
\vdots \\
0\\
1
\end{pmatrix},\quad
\Bb^{(0)} = 0.
\]
Then, the pre-activated output of the first layer, i.e. $\Conv_{\Bw^{(0)},\Bb^{(0)}} (\Bx)$, is 
\[
\begin{pmatrix}
\sum_{i\in [s]} a_i x_i & -\sum_{i\in [s]} a_i x_i & x_s \\
* & * & \vdots \\
\vdots & \vdots & x_d \\
* & * & *
\end{pmatrix} \in \bR^{d\times 3},
\]
where we use $*$ to denote some entries that we do not care. By using the equality $t= \sigma(t)-\sigma(-t)$, we are able to compute the partial inner product $\sum_{i\in [s]} a_i x_i$ after applying the ReLU activation. Since we assume $\Bx\in [0,1]^d$, the $x_i=\sigma(x_i)$, $i\in [s:d]$, are stored in the output of the first layer. For $\ell \in [L-2]$, we define $\Bw^{(\ell)} \in \bR^{s\times 3\times 3}$ and $\Bb^{(\ell)}\in \bR^3$ by
\[
w^{(\ell)}_{:,1,:} = - w^{(\ell)}_{:,2,:} =
\begin{pmatrix}
1 & -1 & 0 \\
0 & 0 & a_{\ell(s-1)+2} \\
\vdots & \vdots & \vdots \\
0 & 0 & a_{(\ell+1)(s-1)+1}
\end{pmatrix},\quad
w^{(\ell)}_{:,3,:} =
\begin{pmatrix}
0 & 0 & 0 \\
\vdots & \vdots & \vdots \\
0 & 0 & 0 \\
0 & 0 & 1
\end{pmatrix},\quad
\Bb^{(\ell)} = 0.
\]
Then, the pre-activated output of the $\ell+1$-th layer is 
\[
\begin{pmatrix}
\sum_{i\in [(\ell+1)(s-1)+1]} a_i x_i & -\sum_{i\in [(\ell+1)(s-1)+1]} a_i x_i & x_{(\ell+1)(s-1)+1} \\
* & * & \vdots \\
\vdots & \vdots & x_{d} \\
* & * & * 
\end{pmatrix}  \in \bR^{d\times 3}.
\]
For $\ell =L-1$, we let $\Bb^{(L-1)}=(b,0,0)^\intercal$. The filter $\Bw^{(L-1)} \in \bR^{s\times 3\times 3}$ and the pre-activated output $o_L \in \bR^{d\times 3}$ are given below
\[
w^{(L-1)}_{:,1,:} =
\begin{pmatrix}
1 & -1 & 0 \\
0 & 0 & a_{(L-1)(s-1)+2} \\
\vdots & \vdots & \vdots \\
\vdots & \vdots & a_d \\
0 & 0 & 0 \\
\end{pmatrix},\quad
w^{(L-1)}_{:,2,:} = w^{(\ell)}_{:,3,:} = 0, \quad o_L=
\begin{pmatrix}
\Ba^\intercal \Bx+b & 0 & 0 \\
* & 0 & 0 \\
\vdots & \vdots & \vdots \\
* & 0 & 0 
\end{pmatrix}.
\]
Finally, for the output weights $\Bw^{(L)} \in \bR^{d\times 3}$, we let $w^{(L)}_{1,1}=c $ and $w^{(L)}_{i,j}=0 $ otherwise. Then, the output of the CNN is exactly $c\sigma(\Ba^\intercal \Bx+b)$. It is easy to see that $\| (\Bw^{(0)},\Bb^{(0)}) \|\le 1$, $\|\Bw^{(L)}\|_1=|c|$, $\| (\Bw^{(\ell)},\Bb^{(\ell)}) \|\le 3$ for $\ell\in [L-1]$. Hence, for this parameterization, $\kappa(\Btheta) \le 3^{L-1} |c|$.

\subsection{Proof of Lemma \ref{CNN construction}}


Given a parameter $R>0$, which will be chosen later, any function $f\in \NN(N,M)$ can be written as 
\[
f(\Bx) = \frac{M}{R} \sum_{i=1}^N c_i\sigma(\Ba_i^\intercal \Bx +b_i),
\]
where $\sum_{i=1}^N |c_i|(\|\Ba_i\|_1 + |b_i|)\le R$. By Lemma \ref{one neuron}, the function $\Bx \mapsto c_i\sigma(\Ba_i^\intercal \Bx +b_i)$ can be implemented by $\CNN(s,3,L_0,3^{L_0-1}R)$, where $L_0=\lceil \frac{d-1}{s-1} \rceil$. We denote the corresponding parameters by $(\Bw^{(0)}(i),\Bb^{(0)}(i),\dots, \Bw^{(L_0-1)}(i),\Bb^{(L_0-1)}(i),\Bw^{(L_0)}(i))$, where $w^{(L_0)}_{j,k}(i)=0$ except for $j=k=1$. By Proposition \ref{rescaling}, we can further assume that $|w^{(L_0)}_{1,1}(i)|\le 3^{L_0-1}R$ and $\|(\Bw^{(\ell)}(i),\Bb^{(\ell)}(i)) \|\le 1$ for all $\ell \in [0:L_0-1]$. 

In order to compute the summation $\sum c_i\sigma(\Ba_i^\intercal \Bx +b_i)$ in a sequential way, we use the 4th channel in the CNN to store the input $\Bx\in [0,1]^d$, and the 5th and 6th channels to store the partial summations of the positive part $\sum_{c_i>0} c_i\sigma(\Ba_i^\intercal \Bx +b_i)$ and the negative part $\sum_{c_i<0} -c_i\sigma(\Ba_i^\intercal \Bx +b_i)$, respectively. The filters $\Bw^{(\ell)} \in \bR^{s\times 6 \times 6}$ and bias $\Bb^{(\ell)} \in \bR^6$ are defined as follows. For $\ell=0$, 
\[
w^{(0)}_{:,1:3,1} = \Bw^{(0)}(1), \quad  w^{(0)}_{:,4,1} = (1,0,\dots,0)^\intercal, \quad \Bb^{(0)} = (\Bb^{(0)}(1)^\intercal, 0,0,0)^\intercal.
\]
Here and in the sequel, we use zero filters and biases when they are not specific defined. We always use the filter $w^{(\ell)}_{:,4,4}= w^{(\ell)}_{:,5,5} = w^{(\ell)}_{:,6,6} = (1,0,\dots,0)^\intercal$ to store the input and partial summations, except for the output layer. If $\ell = (i-1)L_0 +j$ for some $i\in [N]$ and $j\in [L_0-1]$, then we define
\[
w^{(\ell)}_{:,1:3,1:3} = \Bw^{(j)}(i), \quad \Bb^{(\ell)} = (\Bb^{(j)}(i)^\intercal, 0,0,0)^\intercal.
\]
Recall that the only nonzero element in $\Bw^{(L_0)}(i)$ is $w^{(L_0)}_{1,1}(i)$. For $\ell = iL_0$ for some $i\in [N-1]$, we define
\[
w^{(\ell)}_{:,1:3,4} = \Bw^{(0)}(i+1),\quad \Bb^{(\ell)} = (\Bb^{(0)}(i+1)^\intercal, 0,0,0)^\intercal,
\]
and 
\begin{align*}
w^{(\ell)}_{:,5,1} &= (w^{(L_0)}_{1,1}(i),0,\cdots,0)^\intercal, \quad \mbox{if } c_i>0, \\
w^{(\ell)}_{:,6,1} &= (-w^{(L_0)}_{1,1}(i),0,\cdots,0)^\intercal, \quad \mbox{if } c_i<0.
\end{align*}
Then, the activated output of the $iL_0$ layer is of the form
\[
\begin{pmatrix}
c_i\sigma(\Ba_i^\intercal \Bx +b_i)/w^{(L_0)}_{1,1}(i) &* &* &x_1 &\sum\limits_{c_j>0,j<i} c_j\sigma(\Ba_j^\intercal \Bx +b_j) &\sum\limits_{c_j<0,j<i} -c_j\sigma(\Ba_j^\intercal \Bx +b_j) \\
* &* &* &\vdots &* &* \\
* &* &* &x_d &* &*
\end{pmatrix},
\]
where $*$ denotes some entries that we do not care. It is easy to check that we correctly compute the partial summations $\sum_{c_j>0,j\le i} c_j\sigma(\Ba_j^\intercal \Bx +b_j)$ and $-\sum_{c_j>0,j\le i} c_j\sigma(\Ba_j^\intercal \Bx +b_j)$ in the $iL_0+1$ layer. Finally, for the output layer, i.e. $\ell = NL_0$, we define
\[
\Bw^{(NL_0)} = \frac{M}{R}
\begin{pmatrix}
w^{(L_0)}_{1,1}(N) &0 &0 &0 &1 &-1 \\
0 &0 &0 &0 &0 &0 \\
\vdots &\vdots &\vdots &\vdots &\vdots &\vdots
\end{pmatrix}.
\]
Then, the constructed CNN $f_\Btheta(\Bx)$ implements the function $f(\Bx)$.

In our construction, $\| (\Bw^{(\ell)},\Bb^{(\ell)}) \| =1+ |w^{(L_0)}_{1,1}(i)| \le 1 + 3^{L_0-1}R$ if $\ell = iL_0$ for some $i\in [N-1]$ and $\| (\Bw^{(\ell)},\Bb^{(\ell)}) \| \le 1$ for other $\ell \in [0:NL_0-1]$, and $\|\Bw^{(NL_0)}\|_1 = MR^{-1}(2+|w^{(L_0)}_{1,1}(N)|) \le MR^{-1}(2+3^{L_0-1}R)$. Therefore, the norm constraint of the parameter is 
\[
\kappa(\Btheta) \le \frac{M}{R} (2+3^{L_0-1}R) (1+3^{L_0-1}R)^{N-1}.
\]
If we choose $R= 3^{1-L_0} N^{-1}$, then
\[
\kappa(\Btheta) \le \frac{2M}{R}(1+3^{L_0-1}R)^N \le 3^{L_0+1} NM,
\]
where we use $(1+1/N)^N \le e\le 3$. The proof is completed.

Finally, we give a remark on the construction in the proof of Lemma \ref{CNN construction}. This remark is useful for constructing CNNs to approximate functions of the form $g(f_\Btheta(\Bx))$ with $g:\bR \to\bR$.

\begin{remark}\label{construction remark}
We can replace the output layer (i.e. the parameters $\Bw^{(NL_0)}$) of $f_\Btheta$ by a convolutional layer with parameters $w^{(NL_0)}_{:,1,:} = -  w^{(NL_0)}_{:,2,:} = \Bw^{(NL_0)}$ and $w^{(NL_0)}_{:,i,:} =0$ for $i=3,4,5,6$. Then, the activated output of this CNN (without linear layer) is
\begin{equation}\label{temp output}
\begin{pmatrix}
\sigma(f(\Bx)) & \sigma(-f(\Bx)) &0 &0 &0 &0 \\
* &* &\vdots &\vdots &\vdots &\vdots \\
* &* &0 &0 &0 &0
\end{pmatrix}.
\end{equation}
Thus, we can recover $f(\Bx)$ by using $\sigma(f(\Bx))-\sigma(-f(\Bx))$. The norm of this CNN can also be bounded by $3^{L_0+1} NM$.
\end{remark}

\subsection{Proof of Theorem \ref{app bound}}

Let $L_0=\lceil \frac{d-1}{s-1} \rceil$ and $N = \lfloor L/L_0 \rfloor$. By Lemma \ref{CNN construction} and Proposition \ref{inclusion}, we have the inclusion $\NN(N,M_0) \subseteq \CNN(s,6,L,M)$ for $M_0 = 3^{-L_0-1} N^{-1}M$. If $M\gtrsim L^{\frac{3d+3-2\alpha}{2d}}$, then $M_0 \gtrsim L^{\frac{d+3-2\alpha}{2d}}$ and, by the approximation bound (\ref{shallow nn bound}), 
\[
\sup_{h\in \cH^\alpha(1)} \inf_{f\in \CNN(s,6,L,M)} \|h-f\|_{L^\infty([0,1]^d)} \lesssim N^{-\frac{\alpha}{d}} \lor M_0^{-\frac{2\alpha}{d+3-2\alpha}} \lesssim L^{-\frac{\alpha}{d}},
\]
which completes the proof.

\section{Proof of Theorem \ref{rate regression}}

The proof is based on the following lemma from \cite[Appendix B, Lemma 18]{kohler2021rate}. It decomposes the estimation error of the estimator into generalization error and approximation error, and bounds the generalization error by the covering number of the hypothesis class $\cF_n$.

\begin{lemma}\label{gen bound by covering}
Assume that the condition (\ref{noise assumption}) holds. Let $\widehat{f}_n$ be the estimator (\ref{least squares}) and set $B_n = c_1\log n$ for some constant $c_1>0$. Then, 
\begin{align*}
&\bE_{\cD_n} \left[\|\pi_{B_n}\widehat{f}_n-h\|_{L^2(\mu)}^2\right] \\
\le & \frac{c_2 (\log n)^2 \sup_{\BX_{1:n}}\log (\cN(n^{-1}B_n^{-1}, \pi_{B_n}\cF_n,\|\cdot\|_{L^1(\BX_{1:n})})+1)}{n} + 2 \inf_{f\in \cF_n} \|f-h\|_{L^2(\mu)}^2,
\end{align*}
for $n>1$ and some constant $c_2>0$ (independent of $n$ and $\widehat{f}_n$), where $\BX_{1:n} =(\BX_1,\dots,\BX_n)$ denotes a sequence of sample points in $[0,1]^d$ and $\cN(\epsilon, \pi_{B_n}\cF_n,\|\cdot\|_{L^1(X_{1:n})})$ is the $\epsilon$-covering number of the function class $\pi_{B_n}\cF_n:=\{\pi_{B_n}f,f\in \cF_n\}$ in the metric $\|f-g\|_{L^1(\BX_{1:n})} = \frac{1}{n}\sum_{i=1}^n|f(\BX_i)-g(\BX_i)|$.
\end{lemma}

We now prove Theorem \ref{rate regression} by using Theorem \ref{app bound} to bound the approximation error and using Theorem \ref{cover num bound} to estimate the covering number.

\begin{proof}[Proof of Theorem \ref{rate regression}]
It is easy to see that $\cN(\epsilon, \cF_n,\|\cdot\|_{L^1(X_{1:n})}) \le \cN(\epsilon, \cF_n,\|\cdot\|_{L^\infty([0,1]^d)})$. Notice that the projection $\pi_B$ does not increase the covering number. By Theorem \ref{cover num bound}, we have
\[
\log \cN(\epsilon,\pi_{B_n} \cF_n,\|\cdot\|_{L^\infty([0,1]^d)}) \lesssim L_n \log (L_nM_n/\epsilon).
\]
If $M_n\gtrsim L_n^{\frac{3d+3-2\alpha}{2d}}$, by Theorem \ref{app bound}, we get
\[
\inf_{f\in \cF_n} \|f-h\|_{L^2(\mu)}^2 \lesssim L_n^{-\frac{2\alpha}{d}}.
\]
As a consequence, Lemma \ref{gen bound by covering} implies 
\begin{align*}
\bE_{\cD_n} \left[\| \pi_{B_n} \widehat{f}_n - h \|_{L^2(\mu)}^2\right] &\lesssim \frac{\log^2 n}{n} L_n \log(nL_nM_nB_n) + L_n^{-\frac{2\alpha}{d}} \\
&\lesssim \frac{\log^3 n}{n} L_n  + L_n^{-\frac{2\alpha}{d}}, 
\end{align*}
where we use $L_n, M_n \lesssim \Poly(n)$ and $B_n = c\log n$ in the last inequality. Finally, by choosing $L_n \asymp (n/\log^3 n)^{d/(2\alpha+d)}$, we finish the proof.
\end{proof}

\section{Proofs for Section \ref{sec: classification}}

This section gives the proofs of Theorem \ref{rate hinge} and Theorem \ref{rate logistic}. We will first describe the main idea of the proofs. And the details are given in Subsection \ref{proof of rate hinge} and \ref{proof of rate logistic}, respectively.

\subsection{Sketch of proofs}

Our proofs of Theorem \ref{rate hinge} and Theorem \ref{rate logistic} are based on the following lemma, which is summarized from \cite[Appendix A.2]{kim2021fast}.

\begin{lemma}\label{rate lemma}
Let $\phi$ be a surrogate loss function and $\{\cF_n\}_{n\in \bN}$ be a sequence of function classes. Assume that the random vector $(\BX,Y)\in [0,1]^d \times \{-1,1\}$ and the following regularity conditions hold:
\begin{enumerate}[label=\textnormal{(A\arabic*)},parsep=0pt]
\item $\phi$ is Lipschitz, i.e., there exists a constant $c_1>0$ such that $|\phi(t_1)-\phi(t_2)|\le c_1|t_1-t_2|$ for any $t_1,t_2\in \bR$.

\item There exist a positive sequence $\{a_n\}_{n\in \bN}$ and $f_n\in \cF_n$ such that
\[
\cR_\phi(f_n) = \cL_{\phi}(f_n) - \cL_{\phi}(f^*_\phi) \le a_n.
\]
\item There exists a sequence $\{B_n\}_{n\in\bN}$ with $B_n \gtrsim 1$ such that
\[
\sup_{f\in \cF_n} \|f\|_{L^\infty([0,1]^d)} \le B_n.
\]
\item There exists a constant $\nu\in (0,1]$ such that, for any $f\in \cF_n$,
\[
\bE_{\BX,Y}\left[ (\phi(Yf(\BX)) - \phi(Yf^*_\phi(\BX)))^2 \right] \le c_2 B_n^{2-\nu} \cR_\phi(f)^\nu,
\]
where $c_2>0$ is a constant depending only on $\phi$ and the conditional class probability function $\eta$.
\item There exist a sequence $\{\delta_n\}_{n\in \bN}$ and a constant $c_3>0$ such that 
\[
\log \cN(\delta_n, \cF_n, \|\cdot\|_{L^\infty([0,1]^d)}) \le c_3 n \left(\frac{\delta_n}{B_n}\right)^{2-\nu}.
\]
\end{enumerate}
Let $\epsilon_n \asymp a_n \lor \delta_n$ and $\widehat{f}_{\phi,n}$ be the empirical $\phi$-risk minimizer (\ref{ERM}) over the function class $\cF_n$. Then,
\[
\bP(\cR_{\phi}(\widehat{f}_{\phi,n})\ge \epsilon_n) \lesssim \exp(-c_4 n (\epsilon_n/B_n)^{2-\nu}),
\]
for some constant $c_4>0$. In particular, if $n(\epsilon_n/B_n)^{2-\nu} \gtrsim (\log n)^{1+r}$ for some $r>0$, then
\[
\bE_{\cD_n}\left[ \cR_{\phi}(\widehat{f}_{\phi,n}) \right] \lesssim \epsilon_n.
\]
\end{lemma}

Lemma \ref{rate lemma} provides a systematic way to derive convergence rates of the excess $\phi$-risk for general surrogate losses. For the hinge loss $\phi(t) = \max\{1-t,0\}$ and the logistic loss $\phi(t)=\log(1+e^{-t})$, the condition (A1) is satisfied with $c_1=1$. When the function class $\cF_n$ is parameterized by a convolutional neural network, the covering number bound in the condition (A5) can be checked by using Theorem \ref{cover num bound}. Note that $a_n$ in the condition (A2) quantifies how well $f^*_\phi$ can be approximated by $\cF_n$ in the $\phi$-loss.

For the hinge loss, we know that $f^*_\phi = \sgn(2\eta-1)$. Since $f^*_\phi$ is bounded, we can simply choose $B_n=1$ in the condition (A3). The variance bound in (A4) was established by \cite{wu2005svm} and \cite[Lemma 6.1]{steinwart2007fast}. We summarize their result in the following lemma. It shows that we can choose $\nu = q/(q+1)$ in the condition (A4) of Lemma \ref{rate lemma}.

\begin{lemma}\label{hinge variance}
Assume the noise condition (\ref{Tsybakov}) holds for some $q\in[0,\infty]$. Let $\phi$ be the hinge loss. For any $f:[0,1]^d \to \bR$ satisfying $\|f\|_{L^\infty([0,1]^d)} \le B$, it holds that
\[
\bE_{\BX,Y}\left[ (\phi(Yf(\BX)) - \phi(Yf^*_\phi(\BX)))^2 \right] \le c_{\eta,q} (B+1)^{\frac{q+2}{q+1}} \cR_\phi(f)^{\frac{q}{q+1}},
\]
where $c_{\eta,q}>0$ is a constant depending on $\eta$ and $q$.
\end{lemma}

Under the assumption that $\eta \in \cH^\alpha(R)$, we can use Theorem \ref{app bound} to construct a CNN $h$ that approximates $2\eta-1$. We can then approximate $f^*_\phi = \sgn(2\eta-1)$ by a CNN of the form $g\circ h$, where $g$ is a piece-wise linear function that approximates the sign function. The approximation error can be estimated by using the expression (\ref{hinge caculation}). Under the noise condition (\ref{Tsybakov}), we proves that one can choose $a_n \asymp L_n^{-(q+1)\alpha/d}$ in the condition (A2). Furthermore, Theorem \ref{cover num bound} shows that we can choose $\delta_n \asymp (L_n n^{-1} \log n)^{\frac{q+1}{q+2}}$ in (A5). The trade-off between $a_n$ and $\delta_n$ tells us how to choose the depth $L_n$ and gives the desired rate in Theorem \ref{rate hinge}.

For the logistic loss, the convergence rate in Theorem \ref{rate logistic} can be derived in a similar manner. The following lemma shows that one can choose $\nu = 1$ in the condition (A4) of Lemma \ref{rate lemma}.

\begin{lemma}\label{logistic variance}
Let $\phi$ be the logistic loss. For any $f:[0,1]^d \to \bR$ satisfying $\|f\|_{L^\infty([0,1]^d)} \le B$ with $B\ge 2$, it holds that
\[
\bE_{\BX,Y}\left[ (\phi(Yf(\BX)) - \phi(Yf^*_\phi(\BX)))^2 \right] \le 3B \cR_\phi(f).
\]
\end{lemma}
\begin{proof}
Since $\eta(\BX) = \bP(Y=1|\BX)$, 
\begin{align*}
&\bE_{\BX,Y}\left[ (\phi(Yf(\BX)) - \phi(Yf^*_\phi(\BX)))^2 \right] \\
=& \bE_{\BX,Y}\left[ \log^2 \left( \frac{1+e^{-Yf(\BX)}}{1+e^{-Yf^*_\phi(\BX)}} \right) \right] \\
=& \bE_\BX \left[ \eta(\BX) \log^2 \left( \frac{1+e^{-f(\BX)}}{1+e^{-f^*_\phi(\BX)}} \right) + (1-\eta(\BX)) \log^2 \left( \frac{1+e^{f(\BX)}}{1+e^{f^*_\phi(\BX)}} \right) \right] \\
=& \bE_\BX \left[ \eta(\BX) \log^2 \left( \frac{\eta(\BX)}{\psi(f(\BX))} \right) + (1-\eta(\BX)) \log^2 \left( \frac{1-\eta(\BX)}{1-\psi(f(\BX))} \right) \right],
\end{align*}
where $\psi$ is the logistic function defined by (\ref{logistic}) and we use $f^*_\phi = \log(\frac{\eta}{1-\eta})$ in the last equality. Since $|f(\BX)|\le B$ with $B\ge 2$, we have $\psi(f(\BX)) \in [\psi(-B),\psi(B)] = [\psi(-B),1-\psi(-B)]$ with $\psi(-B)\le e^{-2}$. We can apply the following inequalities proven in Proposition \ref{log2 bound} in Supplementary Materials:
\begin{align*}
\eta \log^2 \left( \frac{\eta}{\psi(f)} \right) &\le 2\log(1+e^B) \left(\eta \log \left( \frac{\eta}{\psi(f)} \right) - \eta + \psi(f)\right), \\
(1-\eta) \log^2 \left( \frac{1-\eta}{1-\psi(f)} \right) &\le 2\log(1+e^B) \left((1-\eta) \log \left( \frac{1-\eta}{1-\psi(f)} \right) + \eta - \psi(f)\right),
\end{align*}
where we omit the variable $\BX$. Thus, 
\begin{align*}
&\bE_{\BX,Y}\left[ (\phi(Yf(\BX)) - \phi(Yf^*_\phi(\BX)))^2 \right] \\
\le & 2\log(1+e^B) \bE \left[ \eta \log \left( \frac{\eta}{\psi(f)} \right) + (1-\eta) \log \left( \frac{1-\eta}{1-\psi(f)} \right) \right] \\
\le & 3B \cR_\phi(f),
\end{align*}
where we use the equality (\ref{logistic caculation}) and $\log(1+e^B) \le 3B/2$.
\end{proof}

Recall that, for the logistic loss, $f^*_\phi = \log(\frac{\eta}{1-\eta})$ and $\eta = \psi(f^*_\phi)$, where $\psi$ is defined by (\ref{logistic}). Since $\eta \in \cH^\alpha(R)$, it can be approximated by a CNN $h$ by using Theorem \ref{app bound}. We further construct a CNN $f=g\circ h$, where $g$ is a function that approximates the mapping $t\mapsto \log t -\log (1-t)$, such that $\psi(f)$ approximates $\eta$ well. By equality (\ref{logistic caculation}), the excess $\phi$-risk $\cR_\phi(f) = \bE[\cD_{KL}(\eta,\psi(f))]$ can be estimated by using the following lemma, which is a modification from \cite[Theorem 3.2]{bos2022convergence}.

\begin{lemma}\label{KL bound}
Let $u\in (0,1/2)$ and suppose the SVB condition (\ref{SVB}) holds for some $\beta \in [0,1]$. If a function $h:[0,1]^d \to [u,1-u]$ satisfies $\| h - \eta\|_{L^\infty([0,1]^d)} \le C u$ for some constant $C>0$, then
\[
\bE_\BX[\cD_{KL}(\eta(\BX),h(\BX))] \le 
\begin{cases}
\frac{2(2-\beta)C_\beta (C+1)^{2+\beta}}{1-\beta} u^{1+\beta}, & \beta< 1, \\
2C_1(C+1)^3u^2\log(u^{-1}), & \beta= 1.
\end{cases}
\]
\end{lemma}
\begin{proof}
Using the inequality $\log t \le t-1$ for $t>0$, we have for any $\Bx\in [0,1]^d$,
\begin{align}
\cD_{KL}(\eta(\Bx),h(\Bx)) &\le \eta(\Bx) \left(\frac{\eta(\Bx)}{h(\Bx)} - 1 \right) + (1-\eta(\Bx)) \left(\frac{1-\eta(\Bx)}{1-h(\Bx)} - 1 \right) \nonumber \\ 
&= \frac{(\eta(\Bx)-h(\Bx))^2}{h(\Bx)(1-h(\Bx))}\le \frac{C^2u^2}{h(\Bx)} + \frac{C^2u^2}{1-h(\Bx)}. \label{temp2}
\end{align}
By assumption, we have $h(\Bx) \ge u$ and $h(\Bx) \ge \eta(\Bx)-Cu$. Notice that, if $\eta(\Bx)-Cu \ge u$, then 
\[
\eta(\Bx)-Cu \ge \eta(\Bx) - \frac{C\eta(\Bx)}{C+1} = \frac{\eta(\Bx)}{C+1}.
\]
Therefore,
\[
\frac{1}{h(\Bx)} \le \frac{1}{u} \mathbbm{1}_{\{\eta(\Bx)<(C+1)u\}} + \frac{C+1}{\eta(\Bx)} \mathbbm{1}_{\{\eta(\Bx)\ge (C+1)u\}}.
\]
By the SVB condition (\ref{SVB}), we get
\[
\bE_\BX[h(\BX)^{-1}] \le C_\beta (C+1)^\beta u^{\beta-1} + (C+1)I_u,
\]
where
\begin{align*}
I_u :&= \int_{\{\eta(\Bx)\ge (C+1)u\}} \frac{1}{\eta(\Bx)} dP_\BX(\Bx) \\
&= \int_0^\infty \bP_\BX\left( \frac{1}{\eta(\BX)} \mathbbm{1}_{\{\eta(\BX)\ge (C+1)u\}}\ge t \right)dt \\
&\le \int_0^{\frac{1}{(C+1)u}} \bP_\BX\left( \eta(\BX) \le \frac{1}{t} \right)dt.
\end{align*}
If $\beta <1$, then by the SVB condition (\ref{SVB}),
\[
I_u \le C_\beta \int_0^{\frac{1}{(C+1)u}} t^{-\beta} dt = \frac{C_\beta (C+1)^{\beta-1}}{1-\beta} u^{\beta -1}.
\]
If $\beta =1$, the SVB condition (\ref{SVB}) implies that $C_\beta\ge 1$ and $ \bP_\BX( \eta(\BX) \le t^{-1} ) \le \min\{1,C_\beta t^{-1}\}$, which leads to
\[
I_u \le \int_0^{C_\beta} 1dt + C_\beta \int_{C_\beta}^{\frac{1}{(C+1)u}} t^{-1}dt \le C_\beta(1+\log (u^{-1})).
\]
Thus, we have given a bound for $\bE_\BX[h(\BX)^{-1}]$. Similarly, one can show that the same bound holds for $\bE_\BX[(1-h(\BX))^{-1}]$. Combining these bounds with (\ref{temp2}) finishes the proof.
\end{proof}

Using Lemma \ref{KL bound}, we prove that one can choose $a_n\asymp L_n^{-(1+\beta)\alpha/d} \log n$ and $B_n\asymp \log n$ in conditions (A2) and (A3) of Lemma \ref{rate lemma}. Since $\nu=1$, Theorem \ref{cover num bound} shows that we can choose $\delta_n \asymp L_n B_n n^{-1} \log n$ in (A5). The trade-off between $a_n$ and $\delta_n$ tells us how to choose the depth $L_n$ and gives the rate in Theorem \ref{rate logistic}.

\subsection{Proof of Theorem \ref{rate hinge}}\label{proof of rate hinge}

Recall that, for the hing loss, $f^*_\phi = \sgn(2\eta-1)$. In order to estimate the approximation bound (A2) in Lemma \ref{rate lemma}, we need to construct a CNN to approximate $\sgn(2\eta-1)$. If $\eta$ is smooth, we can approximate $2\eta-1$ by a CNN $f$ using Theorem \ref{app bound}. The sign function can be approximated by a piece-wise linear function $g$. Hence, we need to construct a CNN to implement the composition $g\circ f$. The following lemma give such a construction. This lemma can be seen as an extension of Lemma \ref{CNN construction}. Recall that we use $\NN(N,M)$ to denote the function class of shallow neural networks (\ref{shallow nn}). We will use the notation $\cN\cN_d(N,M)$ to emphasize that the input dimension is $d$.

\begin{lemma}\label{cnn composition}
Let $s\in [2:d]$, $L_0=\lceil \frac{d-1}{s-1} \rceil$ and $f\in \cN\cN_d(N,M)$. If $g\in \cN\cN_1(K,M_0)$, then there exists $f_\Btheta \in \CNN(s,6,NL_0+K+1,36\cdot 3^{L_0} NMKM_0)$ such that $f_\Btheta(\Bx)=g(f(\Bx))$ for all $\Bx\in[0,1]^d$.
\end{lemma}
\begin{proof}
Using the construction in the proof of Lemma \ref{CNN construction} and Remark \ref{construction remark}, we can construct $NL_0$ convolutional layers such that the activated output is given by (\ref{temp output}) and the norm of these layers is at most $3^{L_0+1} NM$. Thus, we only need to implement the one-dimensional function $g$ by a $d$-dimensional CNN. 

Without loss of generality, for any $R>0$, we can normalize $g$ such that
\[
g(t) = \frac{2M_0}{R} \sum_{k=1}^K c_k\sigma(a_k t+b_k), \quad |a_k|+|b_k|=\frac{1}{2}, \sum_{k=1}^K |c_k|\le R.
\]
The construction is similar to Lemma \ref{CNN construction}. We use 2nd and 3rd channels channels to store the input information $\sigma(f(\Bx))$ and $\sigma(-f(\Bx))$. The 4th and 5th channels are used to store the partial summations of the positive part $\sum_{c_k> 0} c_k \sigma(a_k f(\Bx) +b_k)$ and the negative part $\sum_{c_k<0} -c_k\sigma(a_k f(\Bx) +b_k)$, respectively. These can be done in a way similar to the proof of Lemma \ref{CNN construction}. So, we only give the bias and the filter for the first channel: $\Bb^{(NL_0+k)}=(b_k,0,0,0,0)^\intercal$,
\[
w^{(NL_0+k)}_{:,1,:} =
\begin{pmatrix}
0 & a_k & -a_k & 0 & 0 \\
0 & 0 & 0 & 0 & 0 \\
\vdots & \vdots & \vdots & \vdots & \vdots
\end{pmatrix}.
\]
Then, the activated output of the $NL_0+k$-th layer is 
\[
\begin{pmatrix}
\sigma(a_k f(\Bx)+b_k) & \sigma(f(\Bx)) & \sigma(-f(\Bx)) & \sum\limits_{\substack{c_i> 0 \\ i<k}} c_i \sigma(a_i f(\Bx) +b_i) & \sum\limits_{\substack{c_i< 0 \\ i<k}} -c_i \sigma(a_i f(\Bx) +b_i) \\
* & * & * & * & * 
\end{pmatrix},
\]
where, as before,  $*$ denotes some entries that we do not care. Finally, the output layer is given by $\Bw^{(L)}\in \bR^{d\times 5}$ with $L=NL_0+K+1$ and nonzero entries $\Bw^{(L)}_{1,1}=2M_0c_K/R$ and $\Bw^{(L)}_{1,4}=-\Bw^{(L)}_{1,5}=2M_0/R$. 

The norm of the construed CNN is
\begin{align*}
\kappa(\Btheta)&\le 3^{L_0+1} NM \cdot \frac{2M_0}{R}(2+|c_K|) \prod_{k=1}^{K-1}(1+|c_k|) \\
&\le 12\cdot 3^{L_0} NMM_0 \frac{(1+R)^K}{R} \\
&\le 36\cdot 3^{L_0} NMKM_0,
\end{align*}
where we choose $R=1/K$ and use $(1+1/K)^K \le e\le 3$ in the last inequality.
\end{proof}

We are ready to prove Theorem \ref{rate hinge} by using Lemma \ref{rate lemma}.

\begin{proof}[Proof of Theorem \ref{rate hinge}]
Since $\eta \in \cH^\alpha(R)$ by assumption, we have $2\eta-1\in \cH^\alpha(2R+1)$. For a given $u_n\in(0,1)$, by the approximation bound (\ref{shallow nn bound}), there exists $h\in \NN(\widetilde{N}_n,\widetilde{M}_n)$ with $\widetilde{N}_n \asymp u_n^{-d/\alpha}$ and $\widetilde{M}_n \gtrsim u_n^{-(d+3-2\alpha)/(2\alpha)}$ such that
\begin{equation}\label{temp1}
\|2\eta - 1 -h\|_{L^\infty([0,1]^d)} \le u_n.
\end{equation}
Recall that $f^*_\phi = \sgn(2\eta-1)$ for the hing loss. We define
\[
g(t) = 
\begin{cases}
1 \quad & t\ge u_n,\\
t/u_n \quad & |t|<u_n,\\
-1 \quad &t \le -u_n,
\end{cases}
\]
which can also be written as 
\[
g(t) = u_n^{-1}\sigma(t+u_n) - u_n^{-1}\sigma(t-u_n) - \sigma(0t+1).
\]
Since $J\ge 6$ and $2u_n^{-1}(1+u_n)+1 \lesssim u_n^{-1}$, by Lemma \ref{cnn composition}, there exists $f \in \CNN(s,J,L_n,M_n)$ with $L_n \asymp \widetilde{N}_n \asymp u_n^{-d/\alpha}$ and $M_n \asymp \widetilde{N}_n \widetilde{M}_n u_n^{-1} \gtrsim u_n^{-(3d+3)/(2\alpha)}$ such that $f=g\circ h$ on $[0,1]^d$. Notice that $\pi_1 f=f$ by construction, which implies $f\in \cF_n$. 

Let us denote $\Omega_n:=\{\Bx\in[0,1]^d:|2\eta(\Bx)-1|\ge 2u_n \}$, then $|h(\Bx)|\ge u_n$ on $\Omega_n$ by (\ref{temp1}). As a consequence, $f^*_\phi(\Bx) = \sgn(2\eta(\Bx)-1) = g(h(\Bx)) = f(\Bx)$ for any $\Bx\in \Omega_n$. Using equality (\ref{hinge caculation}), we get
\begin{align*}
\cR_{\phi}(f) &= \bE[|f-f^*_\phi||2\eta-1|] \\
&= \int_{[0,1]^d\setminus \Omega_n} |f(\Bx)-f^*_\phi(\Bx)||2\eta(\Bx)-1| dP_\BX(\Bx) \\
&\le 4u_n \bP_\BX(|2\eta(\BX)-1|\le 2u_n) \lesssim u_n^{q+1}.
\end{align*}
Thus, we have shown that one can choose $a_n \asymp u_n^{q+1}$ and $B_n=1$ in conditions (A2) and (A3) of Lemma \ref{rate lemma}. By Lemma \ref{hinge variance}, we can choose $\nu=q/(q+1)$ in condition (A4) of Lemma \ref{rate lemma}. To select $\delta_n$ in condition (A5) of Lemma \ref{rate lemma}, we apply Theorem \ref{cover num bound} to estimate the entropy of $\cF_n$ (note that $\pi_1$ does not increase the entropy):
\begin{align*}
\log \cN(\delta_n, \cF_n, \|\cdot\|_{L^\infty([0,1]^d)}) &\lesssim L_n \log(L_n M_n^2/\delta_n) \\
&\lesssim u_n^{-d/\alpha} \log(n u_n^{-1} \delta_n^{-1}),
\end{align*}
where we use $M_n \lesssim \Poly(n)$ in the last inequality. This bound shows that we can choose $\delta_n \asymp (n^{-1} u_n^{-d/\alpha} \log n)^{\frac{q+1}{q+2}}$ under the assumption that $u_n^{-1}\lesssim \Poly(n)$. Finally, if we set 
\[
u_n \asymp \left( \frac{\log^2 n}{n} \right)^{\frac{\alpha}{(q+2)\alpha+d}},
\]
then $\epsilon_n \asymp a_n\lor \delta_n \asymp a_n = u_n^{q+1}$ satisfies the condition $n(\epsilon_n/B_n)^{2-\nu} \gtrsim \log^2 n$ in Lemma \ref{rate lemma}, and hence we have 
\[
\bE_{\cD_n}\left[ \cR_{\phi}(\widehat{f}_{\phi,n}) \right] \lesssim \epsilon_n\lesssim \left( \frac{\log^2 n}{n} \right)^{\frac{(q+1)\alpha}{(q+2)\alpha+d}},
\]
which completes the proof. 
\end{proof}

\subsection{Proof of Theorem \ref{rate logistic}}\label{proof of rate logistic}

Recall that we use $\cN\cN_1(N,M)$ to denote the function class of shallow neural networks (\ref{shallow nn}) with one dimensional input.

\begin{lemma}\label{app log}
For any integer $N\ge 3$, there exists $g\in \cN\cN_1(2N,6N)$ such that the following conditions hold
\begin{enumerate}[label=\textnormal{(\arabic*)},parsep=0pt]
\item $|g(t)|\le \log N$ for all $t\in \bR$.

\item $g(t) = -\log N$ for $t\le 0$, and $g(t) = \log N$ for $t\ge 1$.

\item For any $t\in [0,1]$, $|\psi(g(t)) - t|\le 3N^{-1}$, where $\psi$ is the logistic function (\ref{logistic}).
\end{enumerate}
\end{lemma}
\begin{proof}
Let us first construct a function $h$ that approximates the logarithm function on $[0,1]$. For convenience, we denote $t_i=i/N$ for $i\in [0:N]$. We let $h(t) = \log t_1 = -\log N$ for $t\le t_1$ and $h(t)=0$ for $t\ge t_N=1$. In the interval $[t_1,t_N]$, we let $h$ be the continuous piecewise linear function with $N-1$ pieces such that $h(t_i) = \log t_i$ for $i\in [1:N]$. The function $h$ can be written as
\[
h(t) = -\log N + k_1\sigma(t-t_1) - k_{N-1}\sigma(t-t_N) + \sum_{i=2}^{N-1}(k_i-k_{i-1}) \sigma(t-t_i),
\]
where $k_i=N(\log t_{i+1} - \log t_i) \le t_i^{-1}$ is the slope of $h$ on the interval $(t_i,t_{i+1})$. Let us consider the approximation error $f(t) = \log t - h(t)$. By construction, $f(t)\ge 0$ on $[t_1,t_N]$ and $f(t_i)=0$ for $i\in [1:N]$. Observe that $|f'(t)-f'(x)| = |t^{-1}-x^{-1}|\le t_i^{-2}|t-x|$ for $t,x\in [t_i,t_{i+1}]$, $i\in [1:N-1]$. It is well known that such a Lipschitz gradient property implies the following inequality \cite[Lemma 1.2.3]{nesterov2018lectures} 
\begin{equation}\label{lip gradient bound}
|f(x) - f(t) -f'(t)(x-t)| \le \frac{1}{2 t_i^2} |x-t|^2,\quad  x,t\in [t_i,t_{i+1}].
\end{equation}
If the function $f(t)$ attains its maximal value on $[t_i,t_{i+1}]$ at some point  $t_i^*\in (t_i,t_{i+1})$, then $f'(t_i^*)=0$ and hence, by choosing $x=t_i$ and $t=t_i^*$ in (\ref{lip gradient bound}), we get
\[
f(t_i^*)\le \frac{1}{2 t_i^2}|t_i-t_i^*|^2 \le \frac{1}{2 t_i^2 N^2}.
\]
As a consequence, for $t\in [t_i,t_{i+1}]$ with $i\in [1:N-1]$, we have 
\begin{align*}
|e^{h(t)}-t| &= t (1-e^{-f(t)})\le tf(t) \le \frac{t_i + \frac{1}{N}}{2 t_i^2 N^2} \\
&= \frac{1}{2t_iN^2} + \frac{1}{2t_i^2N^3} \le \frac{1}{N},
\end{align*}
where we use $t_iN=i\ge 1$ in the last inequality. Observe that, for $t\in [0,t_1]$, we have $|e^{h(t)}-t| = |t_1 - t|\le N^{-1}$. We conclude that $|e^{h(t)}-t|\le N^{-1}$ holds for any $t\in [0,1]$.

Next, we define $g(t) = h(t)-h(1-t)$, then $g\in \cN\cN_1(2N,M)$, where 
\begin{align*}
M &\le 3k_1 + 3k_{N-1} + 3 \sum_{i=2}^{N-1}(k_{i-1}- k_i) \\
& = 6k_1 \le 6 t_1^{-1} = 6N.
\end{align*}
It is easy to check that conditions (1) and (2) hold. For $t\in [0,1]$, 
\begin{align*}
|\psi(g(t)) - t| &= \left| \frac{e^{h(t)}}{e^{h(t)}+ e^{h(1-t)}} - t \right| \\
&\le \left| e^{h(t)} - t \right| + \left| \frac{e^{h(t)}}{e^{h(t)}+ e^{h(1-t)}} - e^{h(t)} \right| \\
&= \left| e^{h(t)} - t \right| +  \frac{e^{h(t)}}{e^{h(t)}+ e^{h(1-t)}} \left| t - e^{h(t)} + 1-t - e^{h(1-t)} \right| \\
&\le 3N^{-1}, 
\end{align*}
which proves condition (3).
\end{proof}

\begin{proof}[Proof of Theorem \ref{rate logistic}]
As in the proof of Theorem \ref{rate hinge}, for a given $u_n\in (0,1/3)$, there exists $h\in \NN(\widetilde{N}_n,\widetilde{M}_n)$ with $\widetilde{N}_n \asymp u_n^{-d/\alpha}$ and $\widetilde{M}_n \gtrsim u_n^{-(d+3-2\alpha)/(2\alpha)}$ such that
\[
\|\eta -h\|_{L^\infty([0,1]^d)} \le u_n.
\]
Let $g\in \cN\cN_1(2\lceil u_n^{-1} \rceil,6\lceil u_n^{-1} \rceil)$ be a function that satisfies Lemma \ref{app log} with $N=\lceil u_n^{-1} \rceil$. Thus, $|g(t)|\le \log \lceil u_n^{-1} \rceil$ for all $t\in \bR$, and 
\[
|\psi(g(t)) - t|\le 3u_n \quad t\in [0,1].
\]
Since $J\ge 6$, by Lemma \ref{cnn composition}, there exists $f \in \CNN(s,J,L_n,M_n)$ with $L_n \asymp \widetilde{N}_n + \lceil u_n^{-1} \rceil \asymp u_n^{-d/\alpha}$ and $M_n \asymp \widetilde{N}_n \widetilde{M}_n \lceil u_n^{-1} \rceil^2 \gtrsim u_n^{-(3d+3+2\alpha)/(2\alpha)}$ such that $f=g\circ h$ on $[0,1]^d$. If we let $B_n =\log \lceil u_n^{-1} \rceil$, then $\pi_{B_n} f=f$ by construction, which implies $f\in \cF_n$.

Define the function $\widetilde{h}:[0,1]^d \to [0,1]$ by 
\[
\widetilde{h}(\Bx) = 
\begin{cases}
0, & h(\Bx)<0,\\
h(\Bx), &h(\Bx)\in [0,1],\\
1, & h(\Bx)>1.
\end{cases}
\]
By the second condition in Lemma \ref{app log} for the function $g$, we have $f = g\circ h = g \circ \widetilde{h}$. Since $\eta(\Bx) \in [0,1]$ for any $\Bx\in [0,1]^d$, we also have $|\widetilde{h}(\Bx) - \eta(\Bx)| \le |h(\Bx) - \eta(\Bx)|\le u_n$. Therefore,
\begin{align*}
|\psi(f(\Bx)) - \eta(\Bx)| &\le |\psi(g(\widetilde{h}(\Bx))) - \widetilde{h}(\Bx)| + |\widetilde{h}(\Bx) - \eta(\Bx)| \\
& \le 3u_n + u_n =4u_n.
\end{align*}
Since $\psi(f(\Bx)) \in [\psi(-B_n),\psi(B_n)] \subseteq [u_n/2,1-u_n/2]$, Lemma \ref{KL bound} implies that
\[
\cR_\phi(f) = \bE[\cD_{KL}(\eta,\psi(f))] \lesssim u_n^{1+\beta} \log u_n^{-1}.
\]
Thus, we have shown that one can choose $a_n \asymp u_n^{1+\beta} \log u_n^{-1}$ and $B_n =\log \lceil u_n^{-1} \rceil$ in conditions (A2) and (A3) of Lemma \ref{rate lemma}. By Lemma \ref{logistic variance}, we can choose $\nu=1$ in condition (A4) of Lemma \ref{rate lemma}. Using Theorem \ref{cover num bound}, we can estimate the entropy of $\cF_n$ as
\begin{align*}
\log \cN(\delta_n, \cF_n, \|\cdot\|_{L^\infty([0,1]^d)}) &\lesssim L_n \log(L_n M_n^2/\delta_n) \\
&\lesssim u_n^{-d/\alpha} \log(n u_n^{-1} \delta_n^{-1}).
\end{align*}
where we use $M_n \lesssim \Poly(n)$ in the last inequality. This bound shows that we can choose $\delta_n \asymp n^{-1} u_n^{-d/\alpha}B_n \log n$ in condition (A5) of Lemma \ref{rate lemma}, under the assumption that $u_n^{-1} \lesssim \Poly(n)$. Finally, if we set 
\[
u_n \asymp \left( \frac{\log n}{n} \right)^{\frac{\alpha}{(1+\beta)\alpha+d}},
\]
then $\epsilon_n \asymp a_n \lor \delta_n \asymp a_n \asymp \delta_n$ satisfies the condition $n\epsilon_n/B_n \gtrsim \log^2 n$ in Lemma \ref{rate lemma}, and hence we have 
\[
\bE\left[ \cR_{\phi}(\widehat{f}_{\phi,n}) \right] \lesssim \epsilon_n\lesssim \left( \frac{\log n}{n} \right)^{\frac{(1+\beta)\alpha}{(1+\beta)\alpha+d}} \log n,
\]
which completes the proof. 
\end{proof}

\bibliographystyle{siamplain}
\bibliography{references}